\documentclass[twoside]{article}

\PassOptionsToPackage{}{natbib}

\usepackage[accepted]{aistats2023_preprint}
\usepackage[round]{natbib}

\usepackage[utf8]{inputenc} 
\usepackage[T1]{fontenc}    
\usepackage{url}            
\usepackage{booktabs}       
\usepackage{amsfonts}       
\usepackage{nicefrac}       
\usepackage{microtype}      
\usepackage{xcolor}         
\title{Local - (learned) Global Mixed Kernels}

\usepackage{amsmath,amsthm,amssymb,sansmath,graphicx,float,color,authblk,stmaryrd,calc}
\usepackage[utf8]{inputenc}
\usepackage[inline]{enumitem}
\usepackage[active]{srcltx}
\usepackage{upgreek}
\usepackage{mathrsfs}
\usepackage[utf8]{inputenc}
\usepackage{makeidx}
\usepackage{oldgerm}
\usepackage{mathrsfs}
\usepackage[active]{srcltx}
\usepackage{verbatim}
\usepackage{aliascnt}
\usepackage{array}
\usepackage[textwidth=4cm,textsize=footnotesize]{todonotes}
\usepackage{xargs}
\usepackage{cellspace}
\usepackage{algorithm,algorithmic}
\usepackage[linesnumbered,algo2e]{algorithm2e}
\usepackage{svg}
\usepackage{bbm}
\usepackage{colortbl}

\usepackage{tikz}
\usepackage{pgfplots}
\pgfplotsset{compat=newest}
\usetikzlibrary{plotmarks}
\usetikzlibrary{arrows.meta}
\usepgfplotslibrary{patchplots}

\usepackage{aliascnt}

\usepackage{xr-hyper}
\usepackage{hyperref}
\usepackage{cleveref}

\usepackage{autonum}
\makeatletter
\newtheorem{theorem}{Theorem}
\crefname{theorem}{theorem}{Theorems}
\Crefname{Theorem}{Theorem}{Theorems}

\newaliascnt{lemma}{theorem}
\newtheorem{lemma}[lemma]{Lemma}
\aliascntresetthe{lemma}
\crefname{lemma}{lemma}{lemmas}
\Crefname{Lemma}{Lemma}{Lemmas}

\newaliascnt{corollary}{theorem}

\aliascntresetthe{corollary}
\crefname{corollary}{corollary}{corollaries}
\Crefname{Corollary}{Corollary}{Corollaries}

\newaliascnt{proposition}{theorem}

\aliascntresetthe{proposition}
\crefname{proposition}{proposition}{propositions}
\Crefname{Proposition}{Proposition}{Propositions}

\newaliascnt{definition}{theorem}

\aliascntresetthe{definition}
\crefname{definition}{definition}{definitions}
\Crefname{Definition}{Definition}{Definitions}

\newtheorem{assumption}{\textbf{A}\hspace{-3pt}}
\Crefname{assumption}{\textbf{A}\hspace{-3pt}}{\textbf{A}\hspace{-3pt}}
\crefname{assumption}{\textbf{A}}{\textbf{A}}

\newaliascnt{remark}{theorem}

\aliascntresetthe{remark}
\crefname{remark}{remark}{remarks}
\Crefname{Remark}{Remark}{Remarks}

\crefname{example}{example}{examples}
\Crefname{Example}{Example}{Examples}

\crefname{algorithm}{algorithm}{algorithms}
\Crefname{Algorithm}{Algorithm}{Algorithms}

\crefname{figure}{figure}{figures}
\Crefname{Figure}{Figure}{Figures}

\usepackage{caption}
\usepackage{subcaption}
\usepackage{multirow}

\def\rset{\mathbb{R}}

\def\nset{\mathbb{N}}

\def\param{w}
\def\paramscal{\omega}
\def\Algo{\texttt{ASkewSGD}}
\DeclareMathOperator{\clip}{clip}
\def\eqsp{\,}
\def\wrt{w.r.t.}
\def\rme{\mathrm{e}}

\newcommand{\ooint}[1]{\left(#1\right)}
\newcommand{\ccint}[1]{\left[#1\right]}

\def\dist{\operatorname{d}}

\def\objfunc{\ell}
\def\ineqfunc{g}
\def\sign{\operatorname{sign}}
\newcommand{\cZ}{\mathcal{Z}}
\def\ActiveSet{\operatorname{I}}
\newcommand{\argmin}{\mathop{\mathrm{arg\,min}}}
\newcommand{\norm}[1]{\left \lVert #1 \right \rVert}

\newcommand{\rmd}{\mathrm{d}}
\def\sy{\mathsf{y}}

\def\sz{\mathsf{z}}
\def\sv{\mathsf{v}}

\def\PE{\mathbb{E}}

\begin{document}
\addtocontents{toc}{\protect\setcounter{tocdepth}{0}}

%

%

\twocolumn[

\aistatstitle{\Algo\ : An Annealed interval-constrained Optimisation method to train Quantized Neural Networks}

\aistatsauthor{ Louis Leconte \And Sholom Schechtman \And  Eric Moulines }

\aistatsaddress{ LISITE, Sorbonne University\\
Mathematical and Algorithmic Sciences Lab (Huawei) \And  CITI, T\'el\'ecom SudParis \And CMAP, Ecole Polytechnique } ]

\begin{abstract}
In this paper, we develop a new algorithm, Annealed Skewed SGD - \Algo\ - for training deep neural networks (DNNs) with quantized weights. First, we formulate the training of quantized neural networks (QNNs) as a smoothed sequence of interval-constrained optimization problems. Then, we propose a new first-order stochastic method, \Algo, to solve each constrained optimization subproblem. Unlike algorithms with active sets and feasible directions, \Algo\ avoids projections or optimization under the entire feasible set and allows iterates that are infeasible. The numerical complexity of \Algo\ is comparable to existing approaches for training QNNs, such as the straight-through gradient estimator used  in BinaryConnect, or other state of the art methods (ProxQuant, LUQ).
We establish convergence guarantees for \Algo\ (under general assumptions for the objective function). Experimental results show that the \Algo\ algorithm performs better than or on par with state of the art methods in classical benchmarks.
\end{abstract}

\section{Introduction}
The use of deep neural networks (DNNs) on computing hardware such as mobile and IoT devices with limited computational and memory resources is becoming increasingly important. This has led to a growing area of research focused on reducing the model size and inference time of DNNs; in this area, the overall goal is to keep the loss of accuracy below an acceptable level compared to floating-point implementations. These methods include, for example, model pruning, neural architecture search, novel efficient architecture design, and low-rank decomposition.
In this work, we focus on network quantization, where weights and activations are quantized to lower bit widths, allowing for efficient fixed-point inference and reduced memory bandwidth usage; see, for example, \cite{Courbariaux2015,jacob2018quantization,darabi2018bnn+,choukroun2019low,deng2020model,qin2020binary,bhalgat2020lsq, chmiel2021logarithmic} and references therein. Quantized neural networks (QNNs) have attracted many research efforts. Nevertheless, the challenge of closing the accuracy gap between full-precision and quantized networks remains open, especially for extremely low-precision arithmetics (e.g. binary). The task of learning a quantized neural network (QNN) can be formulated as minimising the training loss with quantization constraints on the weights, i.e.,
\begin{equation}
\label{eq:optimization-problem}
\min_{\param \in \mathcal{Q}} \objfunc(\param) \,, \, \objfunc(\param)= \PE_{(x, y) \sim p_{\operatorname{data }}}[\ell(f(x, \param), y)],
\end{equation}
where $\mathcal{Q} \subset \rset^d$ is the set of quantization levels, $d$ is the number of parameters (network weights and biases), $\ell$ is the training loss (e.g. the cross-entropy or square loss), $f(x,w)$ is the DNN prediction function, $p_{\operatorname{data}}$ is the training distribution. The quantization constraints in the above program make it an extremely difficult task: the underlying optimization problem is non-convex, non-differentiable, and combinatorial in nature. Optimization of smooth functions of integer valued variables (and even quadratic ones like the max-cut problem in graph theory) is known to be NP-hard \citep{garey1980computers}.  The challenge is to find algorithms that can produce a sensible approximate solution with a manageable computational effort.
Inspired by mixed-integer nonlinear programming (MINLP) problems, several approaches using geometric, analytic, and algebraic techniques have been proposed to transform the discrete problem into a continuous problem. Examples include the use of global or concave optimization formulations, semidefinite programming, and spectral theory (see e.g. \cite{mitchell1998interior,bussieck2003minlplib,horst2013global,beck2000global,murray2010algorithm}).
However, these types of approaches are doomed to fail in the NN context because the number of parameters is several orders of magnitude larger than for classical MINLP problems.

For large training data sets and number of variables $d$, stochastic gradient-based (first-order) methods for finding minimizers of \eqref{eq:optimization-problem}  are often the only manageable option. Several methods have been proposed which transform the  loss function \eqref{eq:optimization-problem}  into a differentiable surrogate (with possibly an additional penalty term) to "favor" quantized solutions.
The general approach is to introduce real-valued "latent" weights $\param \in \rset^d$ from which the quantized weights  are generated; in the binary case, it is classical to use the $\operatorname{\sign(\cdot)}$ function or a differentiable surrogate thereof.
The simplest method, called BinaryConnect (BC) \citep{Courbariaux2015}, is based on straight-through estimators (STE) that ignore the sign conversion in computing the gradient with respect to the latent weights $\param$.
BC reaches state-of the art performance on elementary classification tasks and is still a competitive baseline method for more sophisticated problems. Extensions of STE has also been used for more general QNN by \cite{chmiel2021logarithmic, sun2020ultra, choi2018bridging, wang2019learning}. 

However, despite its success in NN inference, the STE method does not rely on solid theoretical groundings and may be shown to fail  on simple low-dimensional benchmarks - even with convex objective functions; see \cite{bai2018proxquant} and Section 5.3. We discuss this method and its recent improvements in the paragraph on related works (see below).

\textbf{Contributions}
\begin{itemize}[leftmargin=*,itemsep=0pt]
\item We replace the discrete optimization \eqref{eq:optimization-problem} by an annealed sequence of differentiable inequality constraints that converges to \eqref{eq:optimization-problem} when the annealing parameter goes to $0$.
\item We use a novel first-order algorithm proposed in \cite{muehlebach2021constraints} to solve the relaxed subproblems in the annealed sequence, leading to \Algo.  Unlike classical constrained optimization algorithms, including the projection method or sequential quadratic programming \citep{gill2012sequential}, this approach relies exclusively on local approximations of the feasible set. This local approximation includes only the active constraints, and is guaranteed to be a convex polyhedron even if the underlying constraint set is non convex.  This makes the resulting algorithm easy to implement and also ensures that the descent is not stopped as soon as a new constraint is violated. 
\item We show how \Algo\ can be applied to train QNN. The complexity of the resulting algorithm is similar to that of BC or LUQ~\citep{Courbariaux2015, chmiel2021logarithmic} and ProxQuant \citep{bai2018proxquant}. Our algorithm uses high precision latent weights and uses classical backpropagation to evaluate the gradients.  
\item We provide  convergence guarantees for \Algo. We stress that, as opposed to \cite{muehlebach2021constraints}, no convexity assumption on the objective function or the feasible set is made.

\item We evaluate the performance of \Algo\ on classical computer vision datasets using ConvNets and ResNets. Our experiments show that QNNs trained with \Algo\ achieve accuracy very close to that of their floating-point counterparts, and  outperform or are on par with  comparable baselines.
\end{itemize} 
\paragraph{Related works}
\label{sec:related-works}
We focus on BNN and QNN that replace floating-point multiplication and addition operations with efficient fixed-point arithmetic. We do not consider algorithms that use low-bit computations at the learning stage; see \cite{sakr2018per, chen2020statistical}.
Given the abundance of works, it is impossible to give complete references. We focus mostly on methods used in our benchmarks.

\textbf{Binary NN:} The first attempt to train BNN is  BinaryConnect (BC) \citep{Courbariaux2015,hubara2016binarized} which is the first algorithm to implement Quantization Aware Training (QAT); see  \citep{gholami2021survey, zhao2020review, guo2018survey, nagel2021white} and the references therein. BC uses full precision latent weights. On the forward path, the latent weights are binarized. On the backward path, classical backpropagation is applied to update the latent weights, using a differentiable proxy of the binarization function in the gradient calculation. The most common implementation uses the identity proxy, resulting in the straight-through estimator (STE). 
Although the neural network parameters are highly compressed (and quantization errors can be large), the BC-STE estimator and its numerous recent improvements perform satisfactorily in many benchmarks and have become a de facto standard; see ~\cite{hu2018hashing,faraone2018syq,le2021adaste, anderson2018high}. 


ProxQuant (PQ) \citep{bai2018proxquant}, Proximal Mean-Field (PMF) \citep{ajanthan2019proximal}, Mirror Descent
(MD) \citep{ajanthan2021mirror}, and Rotated Binary Neural Networks (RBNN) \citep{lin2020rotated} formulate the task of training BNNs as a constrained optimization problem and discuss different methods to generate binary weights from real-valued latent weights. All of these methods have in common that they use gradual annealing of the conversion mapping, in the sense that, unlike BC and its variants, the latent weights are not projected onto a finite set of quantization values in the forward path. Instead, a force is applied to gradually push the latent weights to the quantization constraints, in a manner reminiscent of homotopy methods for solving nonlinear systems or penalty barrier in nonlinear optimization.

\textbf{BNN as Variational Inference (VI):} Training binary neural networks can also be approached with VI; see among others~\cite{raiko2015techniques,peters2018probabilistic,roth2019training}.
Instead of optimizing  binary weights, the parameters of Bernoulli distributions are learned using the VI Bayesian learning rule; see e.g.~\cite{khan2021bayesian}. Even if unbiased estimators of the ELBO are available, classical methods like MuProp \citep{gu2016muprop} or REINFORCE with variance-reduction baselines \citep{mnih2014neural} have a prohibitively high variance. The use of  Gumbel-Softmax (GS) trick \citep{jang2016categorical,maddison2017concrete} has been advocated in \cite{meng2020training}, but as noted in \cite[Section~4]{shekhovtsov2021bias} there is an issue in the implementation which paradoxically enables the training. The connections between STE algorithms and their many variants - including MD - and VI methods are further discussed  in \cite{shekhovtsov2021reintroducing}.

\textbf{Quantized NN:} The STE estimator is easily adapted to QNN by adding a projection step onto the set of quantization levels in the forward pass \citep{zhou2016dorefa}; see \citep{choi2018bridging, sun2020ultra, chmiel2021logarithmic} and the references therein.
To mitigate performance loss reported in early work from \cite{zhou2016dorefa}, a number of attempts has been proposed. One possible way is to increase the NN size \citep{zagoruyko2016wide}, or the number of channel for convolution layers \citep{mishra2017wrpn, mcdonnell2018training}.  Knowledge distillation has also been considered with some success \citep{mishra2017apprentice}. A teacher network
(typically very large \citep{liu2020reactnet} and trained in full-precision) is employed to help the QNN training (the student network). 

In QNN, the choice of the quantizer and the normalization of the weights (at each layer) play a key role. Many works have been devoted to the design of non-uniform or statistical (distribution dependent) quantizers; see  \citep{banner2018scalable, hou2018loss, bhalgat2020lsq, liang2021pruning, fournarakis2021hindsight, zhou2017incremental, zhou2018adaptive} and the references therein. Statistical quantizers are often more efficient, but they are more complex to implement and often require fine tuning  \citep{zhang2021differentiable}.

A number of works have considered formulating the quantization problem as an optimization problem \citep{ li2017training, li2016ternary, zhu2016trained, carreira2017model, leng2018extremely, polino2018model}, but the proposed methods rely on  assumptions which may not hold for deep neural networks \citep{guo2018survey}. In \cite{moons2017minimum, yang2017designing,esser2015backpropagation}, the QNN training is tackled as an energy efficiency problem, whereas \cite{gong2019differentiable} propose a Differentiable Soft Quantization (DSQ) to efficiently train QNN.

\textbf{Activation function Quantization:} We have so far described the quantization of the network weights. But an efficient implementation also requires the quantization of the activation functions. For BNN, \citep{kim2016bitwise, hubara2016binarized, Rastegari2016} proposed to use  $\operatorname{sign(\cdot)}$ function, but this approach significantly affects the performance. 
More complex quantization schemes have been considered in \cite{choi2018bridging} alleviating performance degradation. Hybrid formats FP8 \citep{wang2018training} or INT8 \citep{wiedemann2020dithered, banner2018scalable} were successfully employed to achieve a low precision training. Recent works have proposed to jointly optimize the quantization parameters (of weights and activations) and the weights parameters. This task can be done by modifying the learning loss or by minimizing the quantization error \citep{zhu2016trained, zhang2018lq, li2019additive}.

\section{Algorithm derivation}\label{sec:alg_dev}
In this section we first introduce the \citet{muehlebach2021constraints} (MJ) algorithm for smooth constrained optimization, initially proposed in a convex setting. We describe the algorithm in full generality and then show how to adapt the MJ algorithm to the QNN setting.
\paragraph{The MJ algorithm} Consider  the following optimization problem:
\begin{equation}\label{eq:constr_prob}
  \min_{\param \in C} \objfunc(\param), \quad C = \{\param \in \rset^d : \ineqfunc(\param) \geq 0 \}\, ,
\end{equation}
where $\objfunc: \rset^d \rightarrow \rset$ denotes the objective function, $\ineqfunc: \rset^d \to \rset^{n_g}$ define the inequality constraints.
We assume that the feasible set $C$ is non-empty and compact and that the functions $\objfunc$ and $\ineqfunc$ are continuously differentiable. We stress that neither $\objfunc$ nor $C$ are assumed to be convex. Standard solutions to find a local minimizer of~\eqref{eq:constr_prob} use either a projected gradient descent algorithm or ``non-linear'' projection like mirror descent. However, $C$ might have a complicated form, in which case computing the projection on $C$ might require to solve a non-trivial  optimization algorithm in itself (and may fail to be properly defined). The basic idea behind \cite{muehlebach2021constraints}'s proposal is to ``skew'' the search direction in order to force the algorithm to find a minimizer of \eqref{eq:constr_prob} without constraining the sequence $(\param_k)_{k \in \nset}$ to the feasible set. For any $\param \in \rset^d$, define by $\ActiveSet(\param)$ the set of active constraints
\begin{equation}
\label{eq:active-constraibts}
\ActiveSet(\param) = \{ i \in \{1,\dots,n_g\}, g_i(\param) \leq 0\} \eqsp.
\end{equation}
Under mild assumptions (basically, \cite{muehlebach2021constraints} assume that Mangasarian Fromowitz constraint qualification conditions hold everywhere and not on the feasible set only) the tangent and normal cones of $C$ at $\param \in C$ are given by:
\begin{align}
\label{eq:tangent-cone}
T_C(\param)&= \{ v \in \rset^d, \nabla g_i(\param)^\top v \geq 0, \text{for all} \, i \in \ActiveSet(\param) \} \, ,\\
\label{eq:normal-cone}
N_C(\param) &= \{ -\sum\nolimits_{i \in \ActiveSet(\param)} \lambda_i \nabla \ineqfunc_i(\param), \lambda_i \in \rset_+  \} \eqsp.
\end{align}
Moreover, the Karush-Kuhn-Tucker (KKT) conditions hold \cite[Theorem~7.2.9]{borwein2006convex}: if $\param^*$ is a local minimizer of \eqref{eq:constr_prob}, then $\param^* \in \cZ$, where
\begin{equation}\label{eq:critical_set}
  \cZ := \{ \param\in C : 0 \in - \nabla \objfunc(\param) - N_{C}(\param)\} \, .
\end{equation}
The MJ algorithm \citep{muehlebach2021constraints} generates iterates in $\rset^d$ as follows:
\begin{align}\label{eq:alg1}
    \begin{cases}
        \param_{k+1} = \param_k + \gamma_k v_k \quad \\
        v_{k} = \argmin\nolimits_{v \in V_{\alpha}(\param_k)} (1/2) \norm{v + \nabla \objfunc(\param_k) }^2 \, ,
    \end{cases}
\end{align}
where $(\gamma_k)$ is a non-increasing sequence of positive step sizes, $\alpha > 0$ is an hyper-parameter, and the sets $V_{\alpha}(\param)$ are defined as:
\begin{equation}\label{eq:def_valph}
       V_{\alpha}(\param) = \{ v \in \rset^d :  \nabla \ineqfunc_i(\param)^\top v \geq - \alpha \ineqfunc_i(\param) \textrm{ for all  } i \in \ActiveSet(\param)  \} \, .
 \end{equation}
If $\param \in C$ and $i \in \ActiveSet(\param)$, then $g_i(\param) =0$ and thus $V_\alpha(\param)$ reduces to $T_{C}(\param)$. The set  $V_{\alpha}(\param)$ can be considered as an extension of the tangent cone "outside" of the feasible set. Note also that $V_{\alpha}(\param)$, for all $\param \in \rset^d$, is a convex polyhedron whose construction includes only the active constraints.

By construction, whenever $\ineqfunc_i(\param_k) \leq 0$,
 $\nabla \ineqfunc_i(\param_k)^\top v_k \geq - \alpha \ineqfunc_i(\param_k)$. Thus, in Eq.~\eqref{eq:alg1}, the velocity $v_k$  is chosen to match the unconstrained gradient flow $-\nabla \objfunc(\param_k)$ as closely as possible, subject to the velocity constraint $v_k \in V_{\alpha}(\param_k)$ (this is illustrated on a simple example in Figure~\ref{fig:vectorfield}, for different values of $\alpha > 0$).
\begin{figure}
    \centering
    \includegraphics[scale=0.25]{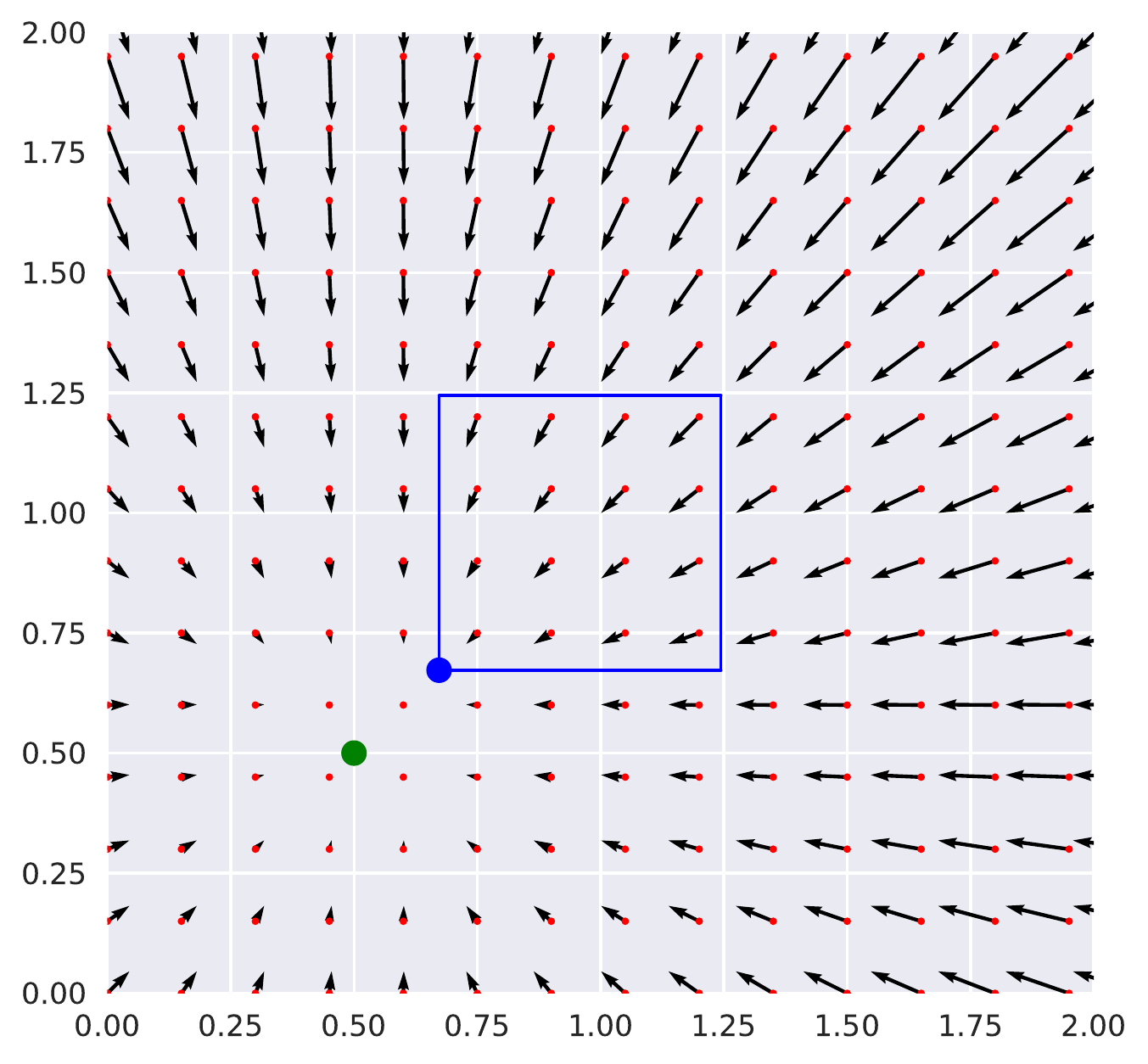}
    \includegraphics[scale=0.25]{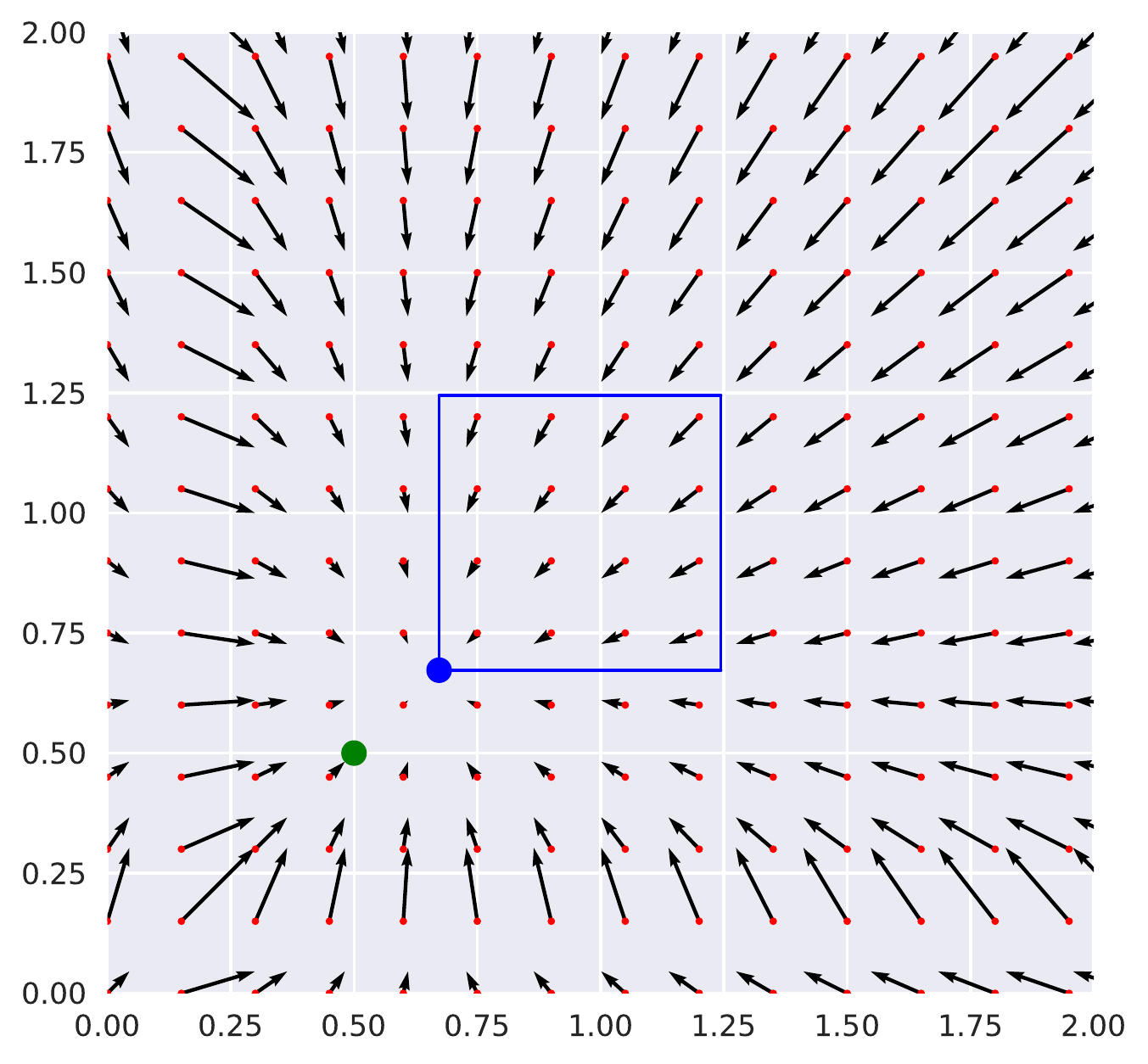}
    \caption{The vector field of velocities for $\epsilon=0.3$, and $\alpha=0.1$ (left panel) or $\alpha=1.0$ (right panel). Here, $\objfunc(\param^1,\param^2) = ((\param^1-0.5)^2+(\param^2-0.5)^2)/3 $ and $\ineqfunc(\param^1,\param^2)^\top= (\epsilon - ((\param^1)^2-1)^2, \epsilon - ((\param^2)^2-1)^2)$. The border of the set of constraints is shown in blue, and the minimizer of the constrained and
unconstrained optimization problems are shown with blue and green dots, respectively.}
    \label{fig:vectorfield}
\end{figure}
A striking difference from the classical projected gradient algorithm is that the MJ approach is based on a local approximation of the feasible set. This local approximation includes only the active constraints and is guaranteed to be a convex polyhedron even if the underlying feasible set is not convex. In ``classical'' constrained optimization algorithms, constraints are typically handled by direct reference to positions, meaning that the iterates $\param_k$, for all $k \geq 0$, must lie in the constraint set $C$.

\paragraph{\Algo\ description}
In \cite{muehlebach2021constraints} the convergence of the MJ algorithm was proven under the condition that the function $\objfunc$ and the set $C$ are convex. We now adapt this algorithm to the QNN problem, removing the requirement that $\objfunc,C$ are convex and, furthermore, replacing $\nabla \objfunc$ by a mini-batch stochastic gradient.

We consider $\objfunc$, the training loss, written as $\objfunc(\param) = 1/N\sum_{j=1}^N \objfunc_j(\param)$, where $N$ is the size of the training set, and $\objfunc_j$ is the loss associated with the $j$-th observation.

We relax the quantization constraints $w^i \in \mathcal{Q}$, $i \in \{1, \dots, d\}$, to a sequence of "smoothed" interval constraints. The set of quantization values $\mathcal{Q}$ is defined coordinate wise: $\{c^i_1,\dots,c^i_{K^i}\}$. We assume for full generality a different scalar quantizer for each coefficient; we typically use different scalar quantizers for each layer of the NN (but the same quantizer for the coefficients in the same layer).
For $\paramscal \in \rset$ such that $\paramscal \in [c^i_1, c^i_{K^i}]$, we define
\[
\phi^i(\paramscal) = (\paramscal - c_{Q^i(\paramscal)}^i)^2(\paramscal - c_{Q^i(\paramscal)+1}^{i})^2
\]
where $Q^i(\paramscal)$ is the unique index satisfying $c^i_{Q^i(\paramscal)}\leq \paramscal < c^{i}_{Q^i(\paramscal)+1}$. If $\paramscal < c^i_1$ we define $\phi^i(\paramscal) =(\paramscal-c^i_1)^2$, and if $\paramscal > c^i_{K^i}$ we define $\phi^i(\paramscal) =(\paramscal-c^i_{K^i})^2$. 
For any $\epsilon \in \ccint{0,1}$ and $\paramscal \in \rset$, define $\psi^i_{\epsilon}(\paramscal):= \epsilon - \phi^i(\paramscal)$ and consider the feasible set
\begin{equation}
\label{eq:definition-C-epsilon}
C_{\epsilon} = \{ \param \in \rset^d : \forall \, i \in \{1, \dots, d\}, \ineqfunc_{\epsilon,i}(\param):= \psi^i_{\epsilon}(\param^i) \geq 0\} \eqsp.
\end{equation}
For each $\epsilon \in \ooint{0,1}$, we  consider the constrained optimization problem $\mathcal{P}_\epsilon : \min_{\param \in C_{\epsilon}} \objfunc(\param)$. It is easily seen that $\bigcap_{\epsilon > 0} C_{\epsilon}= \mathcal{Q}$, recovering  the constraints of the QNN problem. We therefore define a decreasing sequence $(\epsilon_n)_{n \geq 0}$ of numbers in $[0,1]$ such that $\lim_{n \to \infty} \epsilon_n= 0$ and solve (approximately) the sequence of problems $(\mathcal{P}_{\epsilon_n})_{n \in \nset}$.

Here we must notice that the set $V_{\alpha}(\param)$ is empty if and only if there is $ 1\leq i \leq d$ such that $\param^i =(c_{Q^i(\paramscal)}+c_{Q^i(\paramscal)+1})/2$. For such a point, there is no
"best" direction, so we chose it arbitrarily by specifying that the $i$-th coordinate must go to the right (see the following clipping convention). A symmetric choice prescribing a left direction is also possible.  Moreover, since the set of such $\param$ is of Lebesgue-measure
zero, we can hope that we will never stumble upon such a point (this is further guaranteed by the fact that the iterates converge to $C_{\epsilon}$, implying that such points asymptotically never occur).

We denote by $\cZ_{\epsilon} := \{ \param \in C_{\epsilon}: 0 \in - \nabla \objfunc(\param) - N_{C_{\epsilon}}(\param)\}$, the set of KKT points of $\mathcal{P}_{\epsilon}$. Notice that any element of $N_{C_{\epsilon}}(\param)$ can be written as
$(-\lambda^1 \psi'_{\epsilon}(\param^1), \dots,  -\lambda^d  \psi'_{\epsilon}(\param^d))$, with $\lambda^i \geq 0$ and $\lambda^i \neq 0$ only if $\psi_{\epsilon}(\param^i) = 0$.
Therefore, $\param \in \cZ_{\epsilon}$ if and only if for every $i \in \{1, \dots, d\}$,
\begin{align}
    &\nabla_i \objfunc(\param) = 0 &\textrm{if }\psi_{\epsilon}(\param^i) >0 \\
    &\textrm{ and } \sign(\nabla_i \objfunc(\param)) = \sign(\psi'_{\epsilon}(\param^i)) &\textrm{if } \psi_{\epsilon}(\param^i) = 0\, ,
\end{align}
where for $i \in \{1,\dots,d\}$, $\nabla_i \objfunc(\param)$ is the partial derivative of $\objfunc(\param)$ \wrt\ $\param^i$.
In this setting, the set of active constraints and $V_{\alpha}$ can be written down as:
\begin{align}\textstyle
  \ActiveSet_{\epsilon}(\param) &= \{ i \in \{1, \dots, d\} : \psi_{\epsilon}(\param^i) \leq 0 \} \, , \\
     V_{\epsilon,\alpha}(\param) &= \{ v \in \rset^d: v^i \psi_{\epsilon}'(\param^i) \geq - \alpha \psi_{\epsilon}(\param^i) \textrm{ for } i \in \ActiveSet_{\epsilon}(\param)\} \, .
\end{align}
Let $\param$ be such that $\param^i \ne (c_{Q^i(\paramscal)}+c_{Q^i(\paramscal)+1})/2$ for all $i \in \{1,\dots,d\}$. For $u \in \rset^d$, denote by $s_{\epsilon,\alpha}(u,\param) = \argmin_{v \in V_{\epsilon,\alpha}(\param)} 1/2 \norm{v + u}^2$. This problem admits an explicit solution: $[s_{\epsilon,\alpha}(u,\param)]^i = -u^i$ if $\psi_{\epsilon}(\param^i) > 0$  or  $-\psi_{\epsilon}'(\param^i)u^i \geq - \alpha \psi_{\epsilon}(\param^i) \geq 0$
and $[s_{\epsilon,\alpha}(u,\param)]^i= -\alpha\psi_{\epsilon}(\param^i)/\psi_{\epsilon}'(\param^i)$, otherwise.
Note that when $\param^i \rightarrow (c_{Q^i(\paramscal)}+c_{Q^i(\paramscal)+1})/2$, the quantity $\psi_{\epsilon}'(\param^i)$ converges to zero, and thus $[s_{\epsilon,\alpha}(u,\param)]^i $ might diverge to infinity. To alleviate this problem, we furthermore clip the update. For $(a,b) \in \rset \times \rset_{+}$, define $\clip(a,b)$ equal to $a$ if $|a|\leq b$ and to $ b\sign(a)$ otherwise. Choose $M_{\epsilon,c} >0$ and let $s^c_{\epsilon,\alpha}$ be defined for $i \in \{1, \dots, d\}$, $\param^i \ne (c_{Q^i(\paramscal)}+c_{Q^i(\paramscal)+1})/2$, by: $[s^c_{\epsilon,\alpha}(u,\param)]^i = $
\begin{equation}
    \begin{cases}
       -u^i \quad \textrm{ if } \psi_{\epsilon}(\param^i) > 0 \textrm{ or } -\psi_{\epsilon}'(\param^i)u^i \geq - \alpha \psi_{\epsilon}(\param^i) \geq 0\, ;\\
       \clip(-\alpha\psi_{\epsilon}(\param^i)/\psi_{\epsilon}'(\param^i), M_{\epsilon, c}) \quad \quad \quad \quad \textrm{ otherwise}\, .
     \end{cases}
\end{equation}
We set by convention $[s^c_{\epsilon,\alpha}(u,\param)]^i = M_{\epsilon, c}$ if $\param^i = (c_{Q^i(\paramscal)}+c_{Q^i(\paramscal)+1})/2$. For given $\alpha, \epsilon$,  \Algo \, is summarized in Algorithm~\ref{alg:detailed}.
\begin{algorithm}
\caption{\Algo\ algorithm}\label{alg:detailed}
\KwData{sequence of step sizes $(\gamma_k)$; size of the mini-batch $N_b \leq N$; $\param_0 \in \mathbb{R}^d$}
\For{k=1, \dots, T}{
Sample a minibatch of $N_b$ observations $\{j_1, \dots j_{N_b} \}$ in $\{1, \dots, N\}$; \\
Compute the Stochastic Gradient $\widehat{\nabla} \objfunc(\param_k) =1/N_b \sum_{i=1}^{N_b} \nabla \objfunc_{j_i}(\param_k)$; \\
Compute the update direction $v_k = s_{\epsilon, \alpha}^c( \widehat{\nabla} \objfunc (\param_k), \param_k)$; \\
Update the parameter $\param_{k+1} = \param_k + \gamma_k v_k$.
}
\end{algorithm}
Under mild assumptions, we establish the convergence of \Algo. Consider the following assumptions.
\begin{assumption}
\label{assum:sum-objective}
For $j \in \{1, \dots, N\}$, the function $\objfunc_j$ is $d$-times continuously differentiable and has $M_{\objfunc_j}$-Lipschitz continuous gradients.
\end{assumption}
\begin{assumption}
\label{assum:stepsize}
The stepsizes $(\gamma_k)_{k \geq 0}$ are positive, $\sum_{j=0}^\infty \gamma_k = \infty$ and $\sum_{j=0}^\infty \gamma_k^2 < \infty$.
\end{assumption}
Notice that \Cref{assum:stepsize} holds for $(\gamma_k)$ of the form $(1/k^{\delta})$, with $\delta \in (1/2, 1]$.
\Cref{assum:sum-objective} will ensure the stability of \Algo\, (i.e. the iterates are bounded with probability one). Moreover, \Cref{assum:sum-objective} implies that $\objfunc(\cZ_{\epsilon})$ is of empty interior, as a consequence of the Sard's theorem (see Lemma~\ref{lm:sard_part}).
\begin{theorem}\label{th:conv_alg}
Assume \Cref{assum:sum-objective}-\Cref{assum:stepsize} and $0  <  \epsilon \leq \inf_{1 \leq i \leq d} \inf_{1 \leq j \leq {K^i}} |c_j^i - c_{j+1}^i|^4/16$, where $\{c_j^i\}$ are the quantization levels. Then, $\objfunc(\param_k)$ converges and $\lim_{k \to \infty} \dist(\param_k,\cZ_{\epsilon}) = 0$ almost surely.
\end{theorem}
Note that the condition on $\epsilon$ ensures that the projection of $C_{\epsilon}$ onto the $i$-th coordinate is a disconnected set of $K^i$ intervals.
The proof is based on a general convergence result of  \cite{dav-dru-kak-lee-19}, on asymptotic behavior of stochastic approximation of differential inclusion (DI). In our particular case, the corresponding DI is $\dot{\sy}(t) \in - \nabla \objfunc(\sy(t)) - N_{C_{\epsilon}}(\sy(t))$ (we might notice here that this DI is also the continuous-time limit of the projected gradient method). Definitions and important results on DIs and their stochastic approximations can be found in \Cref{sec:app_prel}.

The proof of \Cref{th:conv_alg} is done in several steps (see \Cref{sec:app:proofs} for complete derivations). First we prove that almost surely, the sequence of iterates $(\param_k)$ converges to $C_{\epsilon}$ (see \Cref{lm:ineq_bound}). Then we show that an update step of \Algo\, can be written as $\param_{k+1} = \param_k - \gamma_k \nabla \objfunc(\param_k) + \gamma_k \eta_{k+1} - \gamma_k u_k$, where $\eta_{k+1} = \nabla\objfunc(\param_k) - \widehat{\nabla} \objfunc(\param_k)$ and $u_k$ approximates an element of $N_{C}(\param_k)$.
We show the convergence of $\sum_{j=1}^k \gamma_j \eta_{j+1}$ in \Cref{lm:mart_conv}, and complete the proof by applying \Cref{theo:dav_drus}, which is adapted from \citet[Theorem~3.2]{dav-dru-kak-lee-19}.

\paragraph{Forward pass quantization}
For completeness, we finally describe the quantization of the activation function when \Algo\ is used to train a deep NN. During the forward pass, we employ a  round-to-nearest approach INT4 quantization methods for the activations, taken from \cite{chmiel2021logarithmic}. We make use of Statistics Aware Weight Binning (SAWB) of \citep{choi2018bridging}, which finds the optimal scaling factor that minimizes the quantization
error based on the statistical characteristics of activation distribution. As emphasized by \citep{chmiel2021logarithmic, choi2018bridging}, non-linearities of loss and activation functions make unnecessary the use of an unbiased scalar quantizer. After scaling, we use a uniform quantization (e.g., INT4): the set of quantization values $\mathcal{Q}$ is defined coordinate wise: $\{c^i_1,\dots,c^i_{K^i}\}$. In our experiments (see \Cref{sec:experiments}) both the weights and the activations are rescaled layerwise to fit the quantization interval (e.g., $[-2^3, 2^3]$ for INT4). The quantization values $\{c^i_1,\dots,c^i_{K^i}\}$ are the integer from the quantization interval (e.g., $\{-8, -7,\dots,8\}$ for INT4). After quantization, both the weights and the activation are rescaled using the scaling factor calculated layerwise.
\section{Experiments}
\label{sec:experiments}
We evaluate the performance of \Algo\ with weights quantized with 1, 2, and 4 bits.  While BNN performs well on some simple benchmarks, it lags significantly behind full precision NN on more demanding tasks. QNN with higher precision  and quantization of the activations offers a trade-off between performance and computation efficiency.
For simplicity, we refer to $[\textrm{W}x/\textrm{A}y]$ as a neural architecture with $x$-bit precision weights and $y$-bit precision activations.
Details of the implementations and complementary experiments are reported in \Cref{sec:app:numerical-results}.
In all experiments, $\epsilon$ is annealed throughout the training process during successive episodes.
Our experiments show that the initial value for $\epsilon$ is not critical. We use a logarithmic schedule. Given a fixed $\epsilon$, we run the algorithm until the test error does not improve, and then reduce it by using the last iterates as the starting point for the next round. For example, in the experiments of \Cref{tab:accuracy_cifar10} and \Cref{tab:accuracy_steps}, the initial value for $\epsilon$ is $1$, and we reduce it as $K^t$ with $K=0.88$. We can set $K$ to different values ($\frac{1}{2}$, $0.8$ were tested) as long as $K < 1$.
\subsection{1-bit quantization}
\label{subsec:binarynn}
We evaluate the performance of \Algo\ [W1/A32] on four tasks: a convex problem, a 2D toy example and two classical image classification benchmarks.
\paragraph{Convex toy example}
We compare \Algo\ , BinaryConnect~\citep{Courbariaux2015} and AdaSTE~\citep{le2021adaste} in a logistic regression problem. We generate $n=6000$ feature vectors $\{x_k\}_{k=1}^n$ of dimension $d=10$, drawn independently from the uniform distribution in $[-1,1]$. We randomly choose an optimal vector $w_*$ on the vertices of the hypercube and generate the labels as follows: $y_k \sim \operatorname{Bernoulli}(\{1+\rme^{-x_k^\top w_*}\}^{-1})$.
For completeness, we study how a SGD converges with full precision to the optimal point $w_*$ of this convex problem. All methods are trained for $25$ epochs using the SGD optimizer. The learning rate is set to $1$ and the gradients are computed on random batches of $1000$ samples. For AdaSTE, we have used the code  \footnote{https://github.com/intellhave/AdaSTE} with the hyperparameters specified in the package for annealing. 
\begin{figure}[h!]
    \centering
    \includegraphics[scale=0.45]{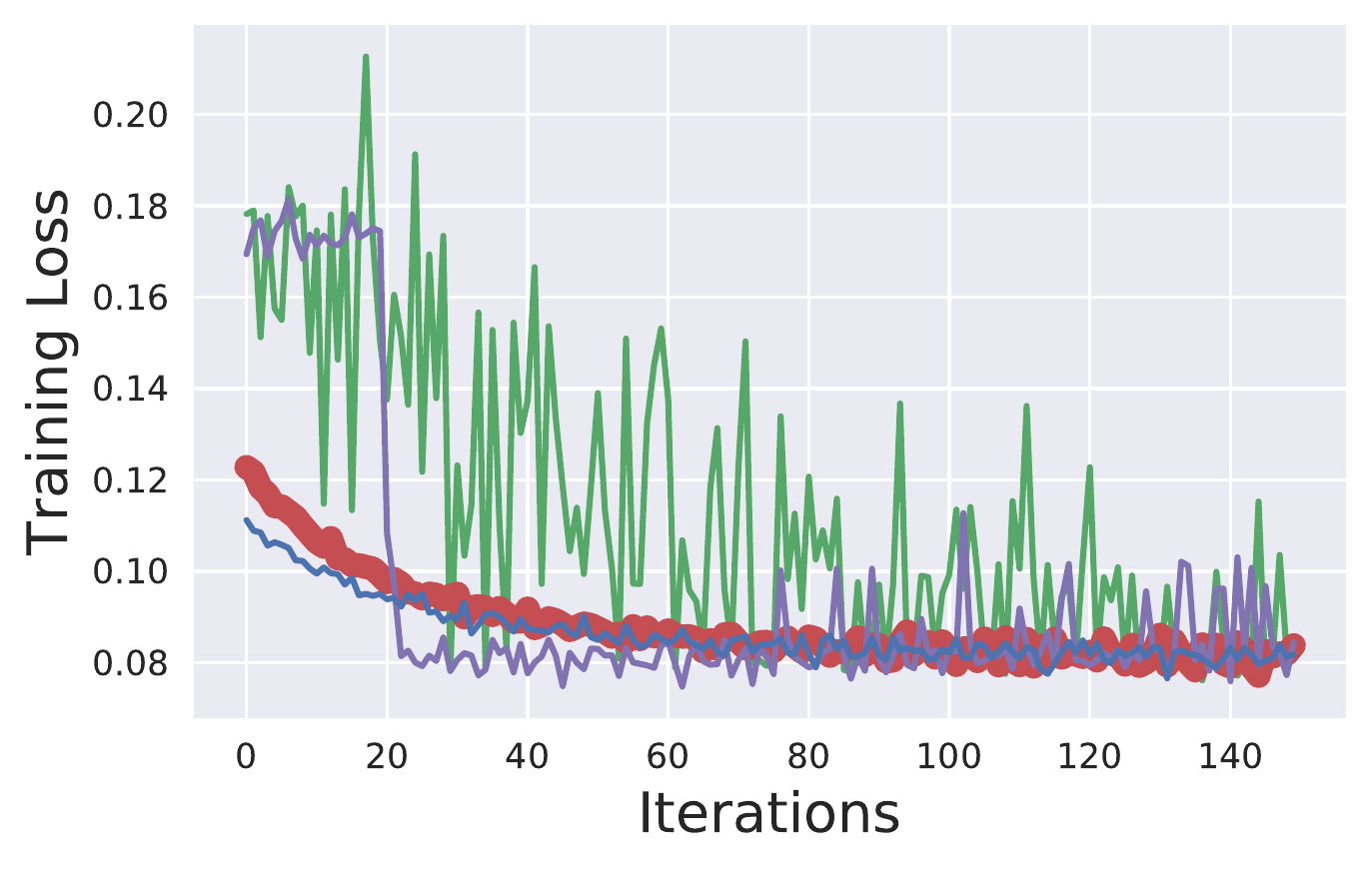}
    \caption{Training losses for the logistic regression problem with batches of size $1000$. BinaryConnect (green), \Algo\ (blue), full Precision (red), AdaSTE (purple) methods. The x-axis represents the iteration index. Red points are made artificially bigger to help visualization.}
\label{fig:loss_logistic_regression}
\end{figure}
 \Algo\ performance is on par with full precision method, while the STE variants all suffer from strong oscillations (see \cite{shekhovtsov2021bias, shekhovtsov2021reintroducing, bai2018proxquant}). \Cref{fig:loss_logistic_regression} illustrates the effects of such oscillations on the convergence . In all settings, AdaSTE converges faster than BC, but still all STE variants exhibit a larger loss compared to other methods. Additional results are reported in \Cref{sec:app:numerical-results}.
 \paragraph{Non-convex toy example}
We consider the binary classification problem on "2 moons dataset" presented in \cite{meng2020training}. The training dataset consists of 2000 samples (split into 2 moon-like clusters in 2 dimensions) and 200 test samples; see \Cref{sec:app:numerical-results}. We train a BNN with 9 neurons. In this low-dimensional environment, we can enumerate all $2^9=512$ possible binary configurations and select the best one(s).
Our method is compared with 4 different approaches: a full precision NN, BinaryConnect~\citep{Courbariaux2015}, AdaSTE~\citep{le2021adaste}, and exhaustive search. All methods are trained for $50$ epochs with logistic loss. The full precision NN is trained using the Adam optimizer \citep{kingma2014adam} with default hyperparameters, a learning rate of $0.1$, and a batch of size $100$. The BinaryConnect approach is trained using the Adam optimizer with default hyperparameters, a learning rate of $1$, and a batch of size $100$. The AdaSTE method is implemented using a learning rate of $1$. Our method uses the same parameters as the STE method, and we set $\alpha$ to $4$. For a single run, we plotted the training loss in \Cref{fig:loss_2moons}.
\begin{figure}[h!]
    \centering
    \includegraphics[scale=0.45]{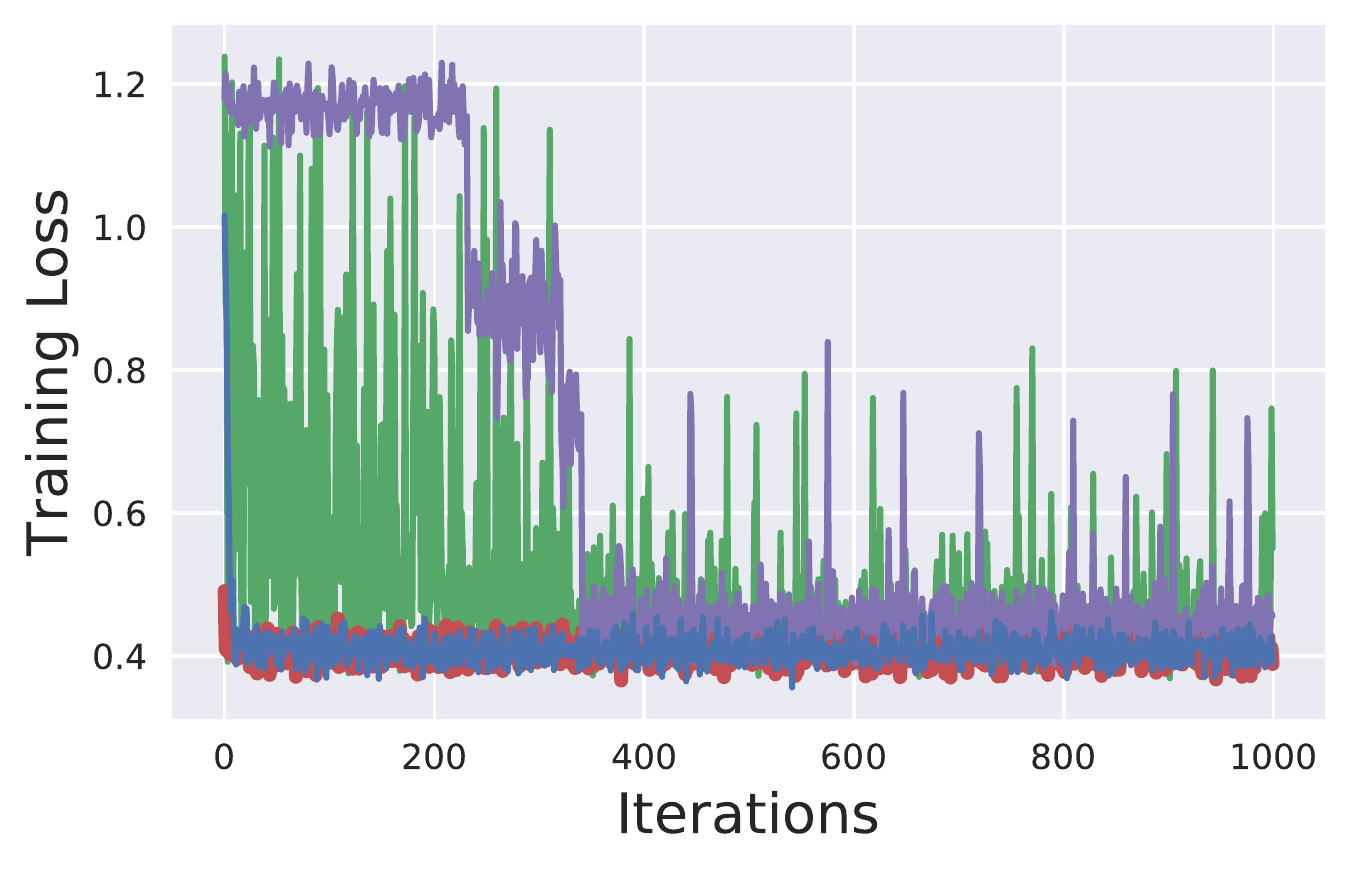}
    \caption{Training losses for the toy non-convex problem with batches of size $100$. BinaryConnect (green), \Algo\ (blue), full Precision (red), AdaSTE (purple) methods. The x-axis represents the iteration index. Red points are made artificially bigger to help visualization}
\label{fig:loss_2moons}
\end{figure}
For a fair comparison, in \Cref{tab:loss_twomoons} in \Cref{sec:app:numerical-results} we report the performance averaged on 50 random experiments of the various methods on the test set (full precision reaches a $2.045 \pm 0.005$ loss, when exhaustive search presents a $2.1$ loss, \Algo\ reaches $2.11 \pm 0.01$, AdaSTE and STE reach $2.24 \pm 0.10$ and $2.32 \pm 0.11$ respectively).

The exhaustive search shows that different configurations lead to near-optimal performance (see \Cref{fig:histogram_bruteforce} in \Cref{sec:app:numerical-results}). Here we chose the configuration that achieves the lowest loss  on the test set.  \Algo\ outperforms AdaSTE and BC.
\paragraph{Computer vision tasks}
In this section, we benchmark \Algo\ with BC \citep{Courbariaux2015,hubara2016binarized}, Mirror Descent \citep{ajanthan2021mirror}, and ProxQuant \citep{bai2018proxquant} on classical computer vision datasets. To avoid overloading the figures, the AdaSTE results are reported separately in \Cref{sec:app:numerical-results}. We also report performance with a standard full precision NN and a full precision NN projected onto the hypersphere. We compare the different methods using the same NN architecture. We do not add bias on any neuron. We introduce batch normalisation (without learning scale and bias parameters) after each layer. We emphasise that our method is generic and not specific to the classical ConvNet architecture. We have also obtained SOTA results for large ResNet architectures (see \Cref{tab:accuracy_tinyimg_imgnet}).

We use the  standard data augmentations and normalizations for all the methods. \Algo\ is implemented in Pytorch, and the experiments are run on a NVIDIA Tesla-P100 GPU. Standard multiclass cross-entropy loss is used for all experiments unless otherwise stated. We perform cross-validation of the hyperparameters, such as the learning rate, the tradeoff between constraints $\alpha$, the rate of increase of the annealing hyperparameter, and their respective schedules. The search space for tuning the hyperparameters and the final hyperparameters can be found in  \Cref{sec:app:numerical-results}. All models are fine-tuned for $100$ epochs using the Adam \citep{kingma2014adam} optimizer with dynamics of $0.9$ and $0.999$, and batch of size $100$.

The NN with full precision is trained with an initial learning rate of $0.08$. The projected full precision NN  uses a projected gradient algorithm. The same hyperparameters as the "plain" algorithm are used, except that a deterministic projection onto the hypersphere is performed for each iteration $\param_{k+1} = \Pi(\param_k - \gamma_k \widehat{\nabla \objfunc}(\param_k))$. For BinaryConnect, we use the method described in \cite{Courbariaux2015}. For Mirror Descent (MD), we use the code\footnote{https://github.com/kartikgupta-at-anu/md-bnn} from \cite{ajanthan2021mirror} and implement the version $\operatorname{tanh}(\cdot)$ (without annealing and with $\alpha=0.01$ and $\mu =100$ when training). ProxQuant was run with the parameters specified in \cite{bai2018proxquant}. Note ProxQuant does not initially quantize the fully-connected layer, and add full precision biases. For fair comparison we have tested
ProxQuant with all layers binarized. The \Algo\ method is described in \Cref{alg:detailed}. Multiple values for $\alpha$ in $[0.1, 5]$ are considered. The precision threshold $\epsilon$ is decreased from epoch to epoch: it is set to $1$ at the beginning and then exponentially annealed to $.88^t$ in the last 50 epochs, where $t$ is the epoch. After the last step, all weights are within an interval of length $\epsilon_{\operatorname{final}}=0.01$ of $\{-1, +1\}$.

For \Algo\, we apply the function $\operatorname{sign}(\cdot)$ to our NN before evaluating it on the test set. For a fair comparison, each method was randomly initialized and independently executed 5 times. An intensive learning rate search was also performed independently for each method.  The learning rate at epochs $[20, 40]$ is divided by $2$ for all methods.

Most neural networks use the inference accuracy of image classification as an evaluation metric. We first compared the training/testing accuracy with the CIFAR-10 dataset \citep{Krizhevsky2009}, which consists of 50000 training images and 10000 test images (in 10 classes). \Cref{fig:histograms_weights} illustrates the distribution of the weights of the first convolutional layer (the behavior is similar for other layers) at epochs 20, 39, 55, and 99. 
\begin{figure}[h!]
    \centering
    \includegraphics[scale=0.25]{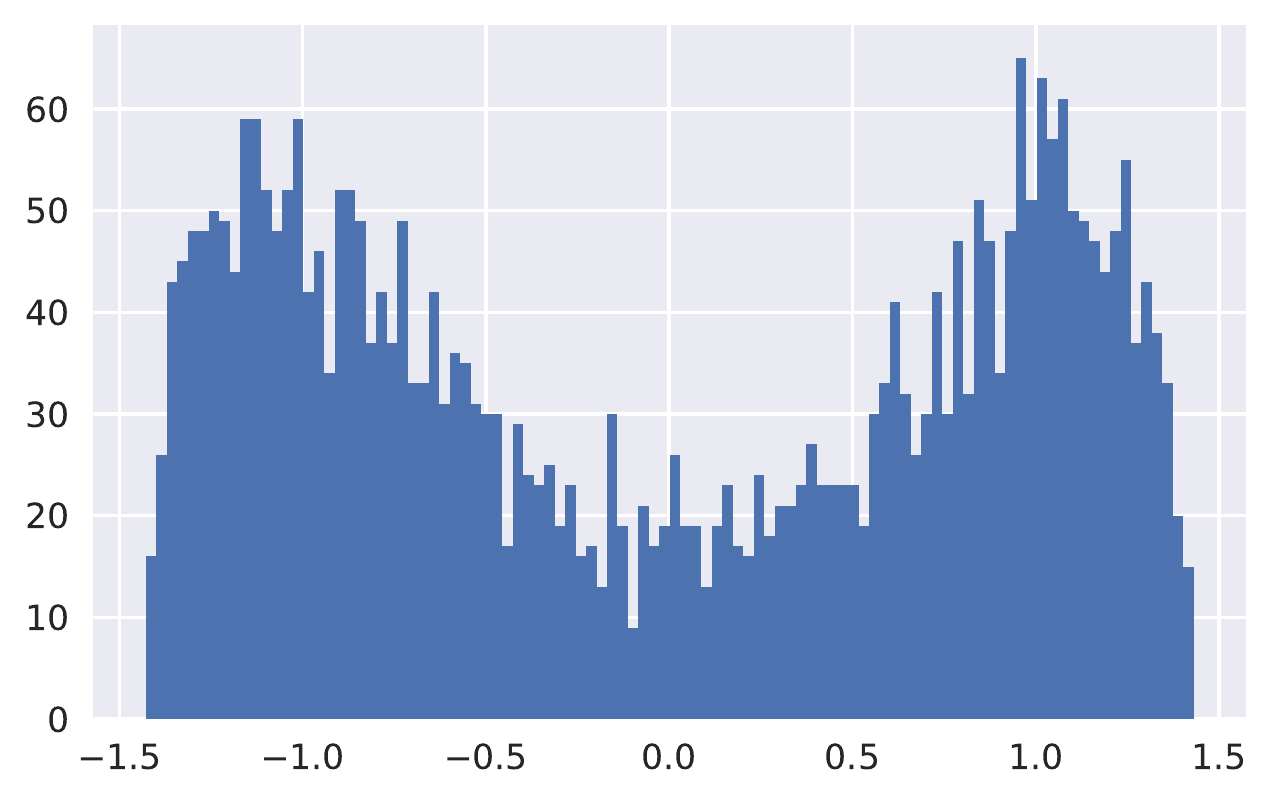}
    \includegraphics[scale=0.25]{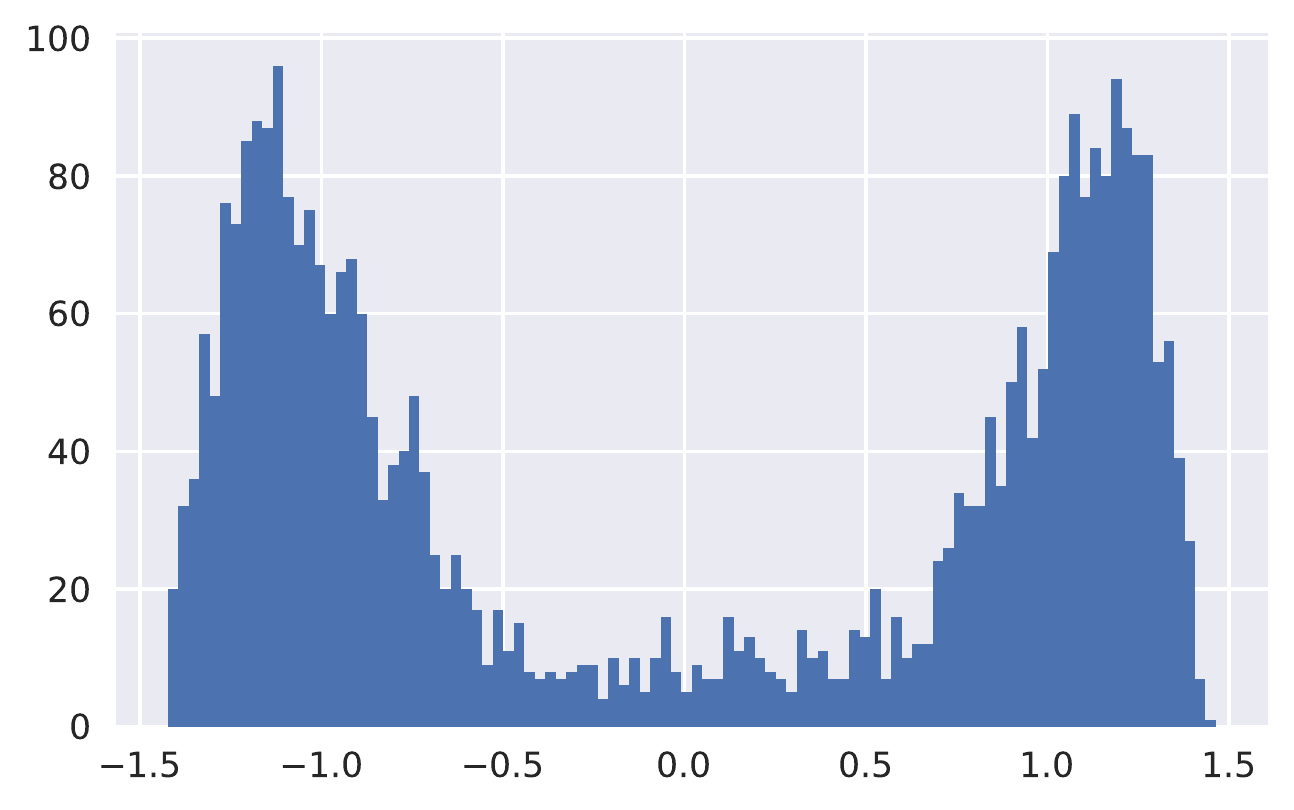}
    \includegraphics[scale=0.25]{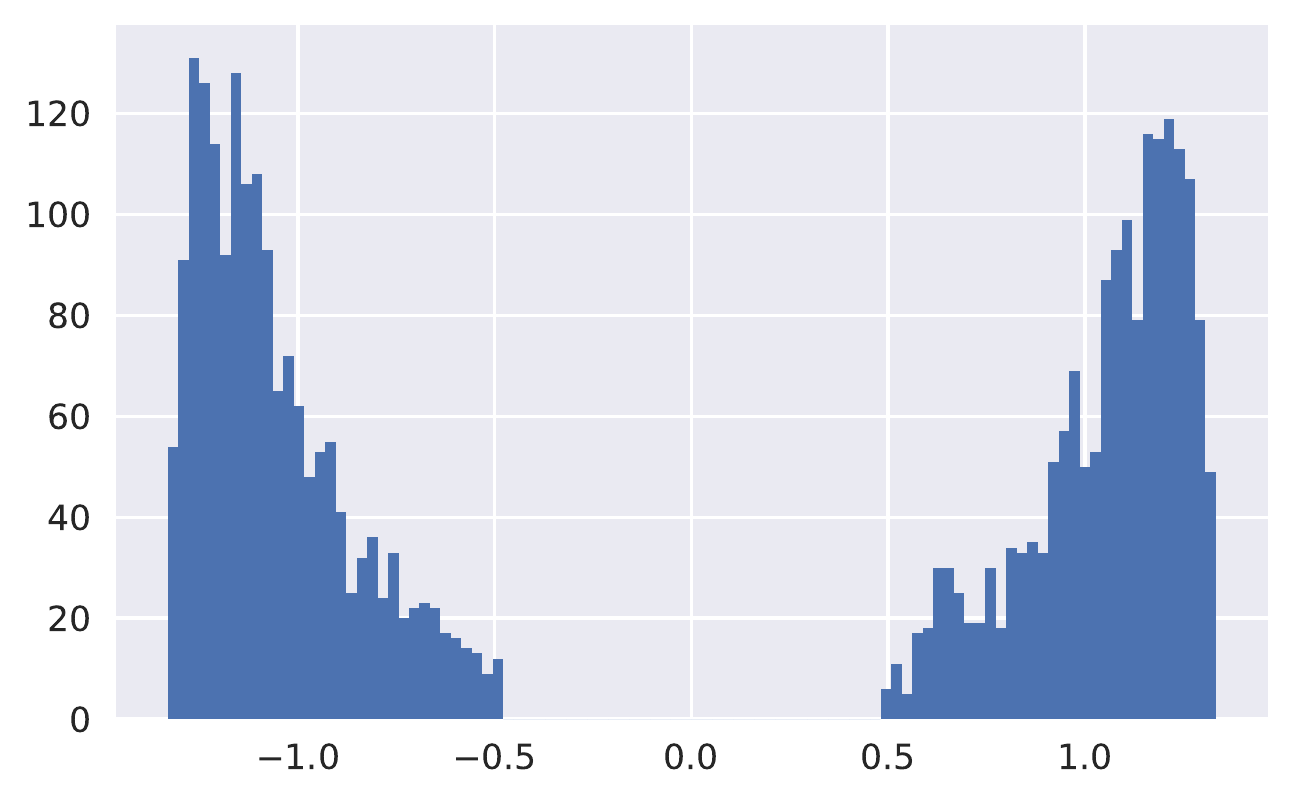}
    \includegraphics[scale=0.25]{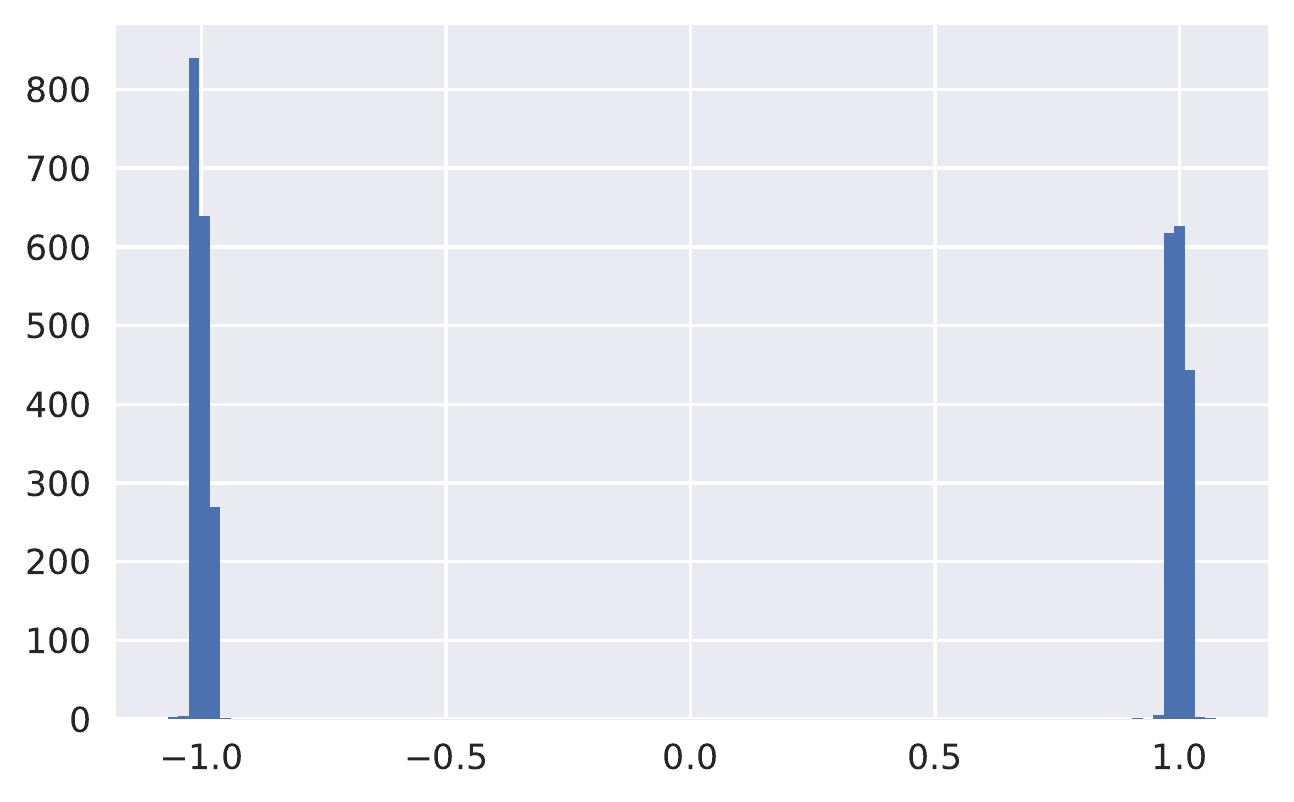}
    \caption{Histogram of weights during the training phase of our \Algo\ [W1/A32] on CIFAR-10.}
    \label{fig:histograms_weights}
\end{figure}
\begin{table}[h!]
  \caption{Test accuracy (average over 5 random experiments) for \Algo\ [W1/A32] at several epochs.}
  \label{tab:accuracy_steps}
  \centering
  \begin{tabular}{cccc}
    \toprule
    Epochs & $\epsilon$ & CIFAR-10 & TinyImageNet \\
    \midrule
    \cellcolor{green!60!black} 50 & 0.88 & 75.77 & 8.74 \\
    \cellcolor{green!80!black} 65 & 0.15 & 88.37 & 31.97 \\
    \cellcolor{green!80!} 90 & 0.006 & 88.84 & 46.96 \\
    \bottomrule
  \end{tabular}
\end{table}
\begin{figure}[h!]
    \centering
    \includegraphics[scale=0.28]{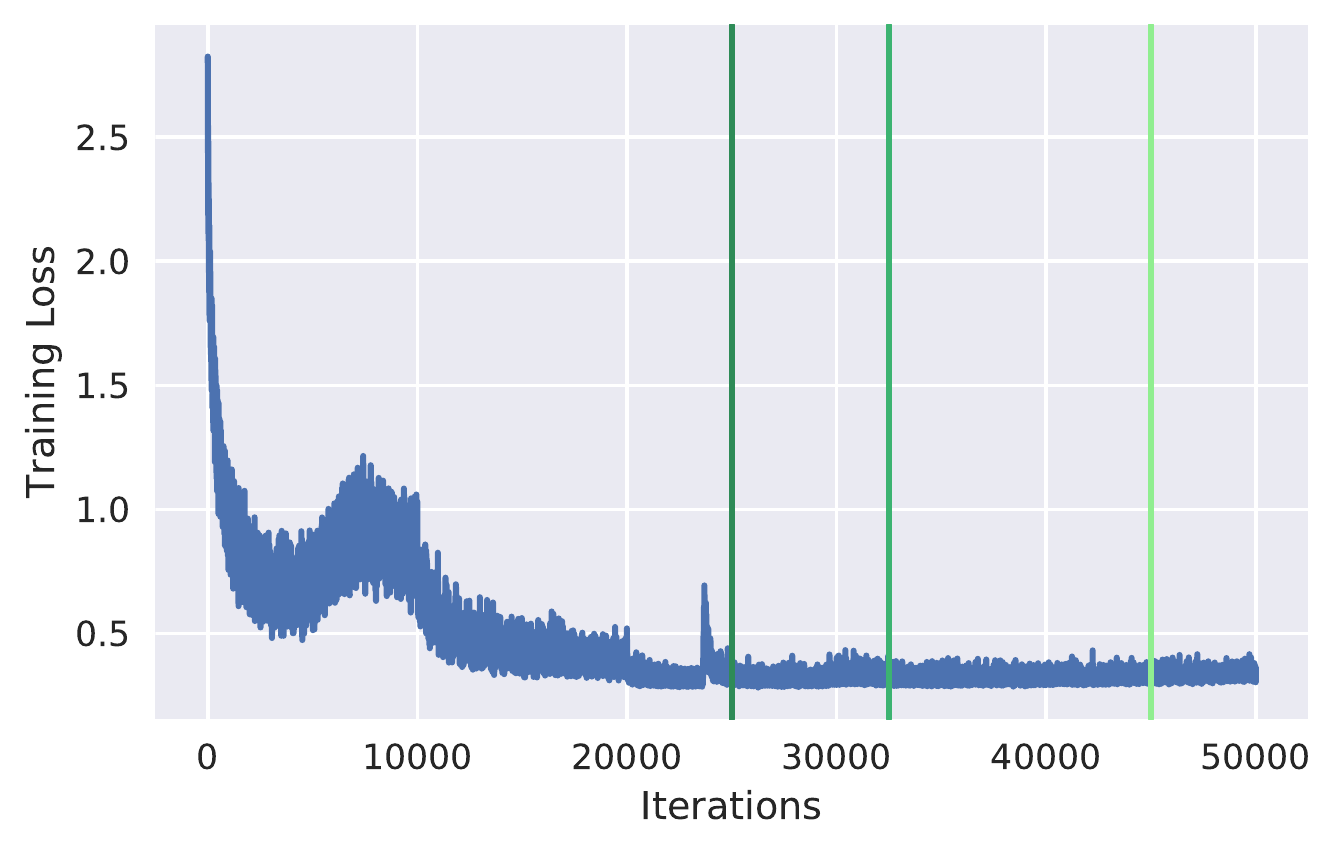}
    \includegraphics[scale=0.28]{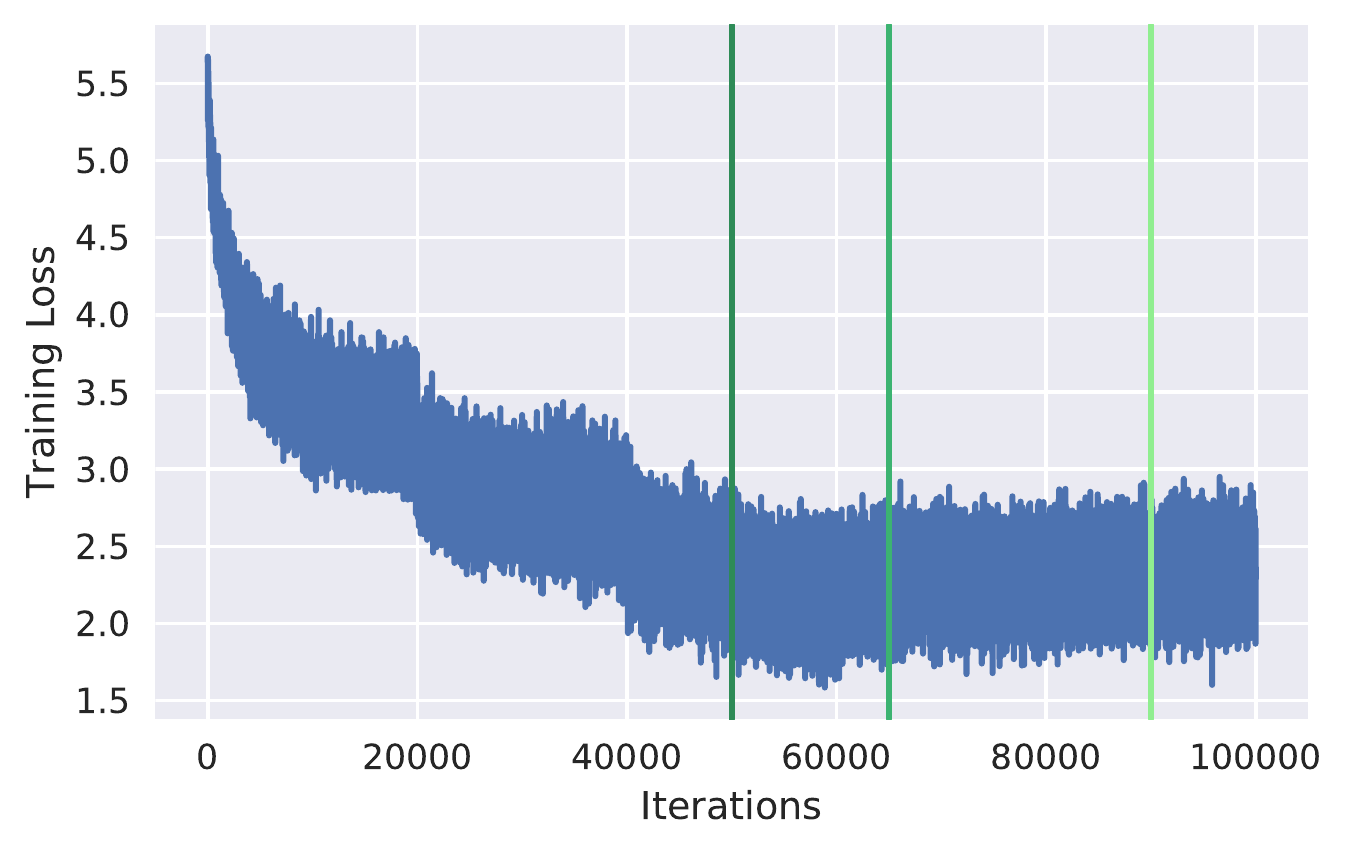}
    \caption{Training Loss of \Algo\ [W1/A32] on CIFAR-10 (left) and TinyImageNet (right). The x-axis represents the batch iterations and green vertical lines correspond to epochs $[50, 65, 90]$.}
    \label{fig:loss_cifar10}
\end{figure}
We have also tested \Algo\ [W1/A32] on the TinyImageNet dataset \citep{le2015tiny} with a ResNet-18. TinyImageNet has 200 classes and each class has 500 (RGB) training images, 50 validation images, and 50 test images. To train ResNet-18 we follow the common practices used for training NNs: we resize the input images to $64 \times 64$ and then randomly flip them horizontally during training. During testing we center-crop them to the corresponding sizes.
In \Cref{fig:loss_cifar10}, the loss increases slightly in the final steps as the constraints become more stringent. However, this increase in training loss remains moderate and the final performance in both the training set and the test set is the best among all methods. Some test accuracies are presented in \Cref{tab:accuracy_steps} at several epochs (identified with green lines in \Cref{fig:loss_cifar10}) with the corresponding precision $\epsilon$. The best test classification accuracies of the binary networks obtained with each method are listed in Table\ref{tab:accuracy_cifar10}. For reproducibility none of the concurrent results are reported from existing papers, but each approach has been independently rerun from the available codes. Compared to other binarization algorithms, our method consistently yields better or equivalent results, while narrowing the performance gap between binary networks and floating-point counterparts on multiple datasets to an acceptable level. The performance of the projected gradient method highlights the strength of our method: we do not simply project the iterates on the nearest constraint set, but progressively push the iterates towards a smoothed version of the constraints (see Section~\ref{sec:alg_dev}), which leads to better results.

\begin{table}[h!]
  \caption{Best Test accuracy (average and variance over 5 random experiments) after 100 training epochs.}
  \label{tab:accuracy_cifar10}
  \centering
  \begin{scriptsize}
  \begin{tabular}{ccc}
    \toprule
    Method & CIFAR-10 & TinyImageNet \\
    \midrule
    Full-precision [W32/A32] & 89.46 $\pm$ 0.07 & 56.46 $\pm$ 0.46 \\
    \midrule
    BinaryConnect [W1/A32]& 88.33 $\pm$ 0.29 & 42.35 $\pm$ 0.33 \\
    MD [W1/A32]& 88.13 $\pm$ 0.25 & 34.89 $\pm$ 0.36 \\
    ProxQuant [W1/A32]& 88.22 $\pm$ 0.28 & 48.79 $\pm$ 0.32 \\
    \midrule
    Projected gradient [W1/A32]& 71.34 $\pm$ 0.46 & 11.78 $\pm$ 0.67\\
    \Algo\ [W1/A32]& \textbf{88.98} $\pm$ 0.35 & \textbf{50.23} $\pm$ 0.37 \\
    \bottomrule
  \end{tabular}
  \end{scriptsize}
\end{table}

\subsection{Low-bit quantization}
We consider now low-bit weight quantization and activation quantization.  To fully benefit from low precision arithmetic, one should also tackle the problem of gradient quantization \citep{chmiel2021logarithmic, sun2020ultra} and accumulation.  We keep the last fully connected layer in full-precision, following \cite{liu2020reactnet, chmiel2021logarithmic}.  We evaluate the performance of \Algo\ [W1/A32], \Algo\ [W2/A4], and \Algo\ [W4/A4] on  TinyImageNet and ImageNet \citep{ILSVRC15} datasets with a ResNet-18 network. For \Algo\, we project NN weights onto the set of quantization values before evaluating it on the test set. For ImageNet, we  keep the first convolution layer in full-precision. We use the same pre-processing (centering and data normalization) for all the methods: we resize the input images to $256 \times 256$ and then randomly crop them to $224 \times 224$ while centering them to the appropriate sizes during training. Standard multiclass cross entropy loss is used. All models are fine-tuned for $200$ epochs using the Adam \citep{kingma2014adam} optimizer with dynamics of $0.9$ and $0.999$ and a batch of size $512$. All methods are trained with an initial learning rate of $0.06$ for TinyImagenet and $0.1$ for ImageNet. The same hyperparameters are used as in the previous section for TinyImageNet. For ImageNet, the learning rate at epochs $[30, 60, 90]$ is divided by $10$ for all methods. 
We have run the code\footnote{https://openreview.net/forum?id=clwYez4n8e8} from LUQ and adapted it to TinyImageNet dataset. For a fair comparison we compute neural gradients in full precision. The results for the method Ultra-low \citep{sun2020ultra} are taken  from \cite{chmiel2021logarithmic}.

We decided not to include the regularisation-based binarization approach \citep{ding2019regularizing}, which addresses the activation binarization problem, in our benchmark. We have also not included in our benchmark  improvements of BC methods which have been proposed in \citep{zhou2016dorefa, liu2018bi, bethge2020meliusnet, Rastegari2016, martinez2019training}; these methods  are all based on the STE~\citep{Courbariaux2015} optimizer to update quantized weights. These methods have been shown to be outperformed by AdaBin \citep{tu2022adabin} and ReacNet \cite{liu2020reactnet}. The latter are currently SOTA methods for energy-friendly inference on the ImageNet dataset. Note that these binary approaches still have a gap in terms of full precision performance, which needs to be addressed by modifying the NN structure \citep{liu2020reactnet}. For ReacNet and AdaBin, we have reported the best results of \cite{tu2022adabin} for ResNet-18 on ImageNet.

\begin{table}[h!]
  \caption{Best Test accuracy (single run for ImageNet due to longer training time) after 200 training epochs. * indicates the results are directly reported from existing literature.}
  \label{tab:accuracy_tinyimg_imgnet}
  \centering
  \begin{tabular}{ccc}
    \toprule
    Method & TinyImageNet & ImageNet\\
    \midrule
    Full-precision [W32/A32] & 56.46 $\pm$ 0.46 & 69.32\\
    \midrule
    ReacNet [W1/A1] (2 steps) & - & 65.5*\\
    AdaBin [W1/A1] (2 steps) & - & 66.4*\\
    \midrule
    Ultra-low [W4/A4] & - & 68.27*\\
    LUQ [W2/A4] & 54.14 $\pm$ 0.42 & - \\
    LUQ [W4/A4] & 55.69 $\pm$ 0.32 & 68.41 \\
    \Algo\ [W2/A4] & 53.54 $\pm$ 0.28 & 66.45\\
    \Algo\ [W4/A4] & \textbf{55.85} $\pm$ 0.30 & \textbf{68.51}\\
    \bottomrule
  \end{tabular}
\end{table}
\Algo\ performs better than or on par with state of the art QNN methods and offers a shorter gap to full precision performances compared with best BNNs.
\section{Conclusion}
In this paper, we present \Algo\ a novel framework for QNN training based on an annealed sequence of interval-constrained nonconvex optimization problems solved by an algorithm inspired by \cite{muehlebach2021constraints}. For each of these subproblems we give theoretical guarantees. \Algo\ outperforms or is on par with other QNN training methods on all considered tasks.

In the current context, we estimated the carbon footprint of our experiments to be about 180 kg CO2e (calculated using green-algorithms.org v2.1 \cite{lannelongue2021green}). This shed light on the crucial need to develop energy friendly NNs.

\subsubsection*{Acknowledgements}
The authors acknowledge support of the Lagrange Mathematics and Computing Research Center.

\clearpage
\pagebreak
\bibliographystyle{apalike}
\bibliography{biblio}


\clearpage
\pagebreak
\appendix

\section{Proofs of \Cref{sec:alg_dev}}
\label{sec:app:proofs}
\subsection{Preliminaries}\label{sec:app_prel}
\textbf{Absolutely continuous curves.}
We say that a curve $\sy: \rset_{+} \rightarrow \rset^d$ is absolutely continuous (a.c.) if there is a curve $\sz: \rset_{+} \rightarrow \rset^d$, locally Lebesgue integrable, such that for every $t \geq 0$,
\begin{equation}
  \sy(t) - \sy(0) = \int_{0}^t \sz(u) \rmd u\, .
\end{equation}
In this case, it holds that for almost every $t \geq 0$, $\sy$ is differentiable and $\dot{\sy}(t) = \sz(t)$.

\textbf{Tangent and normal cones.} Let $C \subset \rset^d$ be a closed set. For $\param \in C$, the tangent cone of $C$ to $\param$, denoted by $T_{C}(\param)$, is the set of vectors $v \in \rset^d$ for which there exist $t_k \downarrow 0$ and $\param_k \rightarrow \param$, $\param_k \in C$, such that $(\param_k - \param)/t_k \rightarrow v$. The normal cone of $C$ at $\param$, denoted $N_{C}(\param)$, is the set of vectors $u \in \rset^d$ such that for any $v \in T_{C}(\param)$, $u^{\top}v \leq 0$. If $\param \notin C$, then by convention $T_{C}(\param), N_{C}(\param) = \emptyset$.

\textbf{The Mangasarian-Fromovitz constraint qualification (MFCQ) condition.}
Consider the case where $C = \{ \param \in \rset^d: \ineqfunc(\param) \geq0\}$, for a smooth function $\ineqfunc: \rset^d \rightarrow \rset^{n_g}$. Denote $\ActiveSet(\param) = \{ i \in \{1,\dots,n_g\}, g_i(\param) \leq 0\}$ as the set of active constraints. We say that the MFCQ condition holds at $\param \in \rset^d$ if there exists $v \in \rset^d$ such that $\nabla g_i(\param)^{\top}v \geq 0$ for all $i \in \ActiveSet(\param)$. If the MFCQ condition holds at $\param \in C$, then we can write down $T_C(\param)= \{ v \in \rset^d, \nabla g_i(\param)^\top v \geq 0, \text{for all} \, i \in \ActiveSet(\param) \}$
and
$N_C(\param) = \{ -\sum_{i=1}^{n_g} \lambda_i \nabla \ineqfunc_i(\param), \lambda_i \in \rset_+ \textrm{ and } \lambda_i=0 \textrm{ if } i \not\in \ActiveSet(\param) \}$ (see, e.g., \cite[Section 7.2]{borwein2006convex}). We might notice here, that in the context of Theorem~\ref{th:conv_alg} the MFCQ condition holds at every $\param \in C_{\epsilon}$.

\textbf{Differential inclusion.} Consider a closed set $C \subset \rset^d$ and $\objfunc: \rset^d \rightarrow \rset$ a smooth function. An essential ingredient of our proof will be the following differential inclusion (DI):
\begin{equation}\label{eq:di}
  \dot{\sy}(t) \in - \nabla \objfunc(\sy(t)) - N_{C}(\sy(t))\, .
\end{equation}
We say that an a.c. curve $\sy : \rset_{+} \rightarrow C$ is a solution to this DI if the inclusion holds for almost every $t \geq 0$. We say that $\objfunc$ is a Lyapunov function for the set $\cZ:= \{ \param \in \rset^d : 0 \in - \nabla \objfunc(\param) - N_C(\param)\}$ if for any such curve:
\begin{equation}
  \textrm{ for all }t>0\, , \quad \objfunc(\sy(t)) \leq \objfunc(\sy(0))\, ,
\end{equation}
with strict inequality as soon as $\sy(0) \notin \cZ$. We have the following lemma.
\begin{lemma}\label{lm:mfcq_lyap}
  Assume that MFCQ holds at every $\param \in C$. Then $\objfunc$ is a Lyapunov function for the DI~\eqref{eq:di} and the set $\cZ$.
\end{lemma}
\begin{proof}
   The assumption that MFCQ holds at every $\param \in C$ implies that $C$ is Clarke regular (i.e. if $(\param_k, u_k) \rightarrow (\param,u) \in C \times \rset^d$ with $(\param_k, u_k) \in C \times N_{C}(\param_k)$, then $u \in N_{C}(\param)$). As shown in (Sections 5 and 6 \cite{dav-dru-kak-lee-19}), this implies that for almost every $t \geq0$ and every $v \in N_{C}(\sy(t))$, $\dot{\sy}(t)^\top v = 0$. Therefore, for almost every $t \geq 0$,
   \begin{align}
     \frac{\rmd}{\rmd t}\objfunc(\sy(t)) &= \nabla \objfunc(\sy(t))^\top \dot{\sy}(t) \\
     &\in - \norm{\dot{\sy}(t)}^2 - \dot{\sy}(t)^\top \sv(t) = - \norm{\dot{\sy}(t)}^2 \, ,
   \end{align}
   where $\sv(t) = \nabla \objfunc(\sy(t)) - \dot{\sy}(t) \in -N_{C}(\sy(t))$.
   This shows that $\sy(t) -\sy(0)= -\int_{0}^t \norm{\dot{\sy}(u)}^2\rmd u $, which, by closedness of $\cZ$, implies our statement.
\end{proof}

In Section~\ref{sec:alg_dev} the set of interest will be $C_{\epsilon}$. It can be easily seen that under the assumptions of Theorem~\ref{th:conv_alg} the MFCQ condition is satisfied at every $\param \in C_{\epsilon}$. Thus, Lemma~\ref{lm:mfcq_lyap} implies that, in this context, $\objfunc$ is a Lyapunov function for the DI: $\dot{\sy}(t) \in -\objfunc(\sy(t)) - N_{C_{\epsilon}}(\sy(t))$.


\textbf{Discrete approximations of differential inclusions.} The idea of our proof is to apply the results of \cite{dav-dru-kak-lee-19} on the stochastic approximation of differential inclusions to our setting. To this end, we consider an $\rset^d$-valued sequence $(y_k)$ constructed as follows:
\begin{equation}
  y_{k+1} = y_k - \gamma_k \nabla \objfunc(y_k) + \gamma_k \eta_{k+1} - \gamma_k u_k \, ,
\end{equation}
where $(\gamma_k)$ is a sequence of positive step-sizes and $(\eta_{k}), (u_k)$ are some $\rset^d$-valued sequences. Here, $u_k$ represent some approximation of an element of $N_{C}(y_k)$, and $\eta_{k+1}$ some (stochastic or deterministic) perturbation. Therefore, $(y_k)$ might be seen as an Euler-like discretization of the DI~\eqref{eq:di}.

The following proposition follows from a general result of \cite[Theorem 3.2]{dav-dru-kak-lee-19}. We state it, applied to our particular case.

\begin{theorem}
\label{theo:dav_drus}
Assume that:
\begin{enumerate}[nosep]
  \item The sequence $(\gamma_k)$ satisfies $\sum_{j=0}^{+\infty} \gamma_j = + \infty$ and $\sum_{j=0}^{+\infty} \gamma_j^2 < +\infty$.
  \item The sequence $\left(\sum\nolimits_{j=0}^{n} \gamma_j \eta_{j+1} \right)$ converges.
  \item The sequence $(y_k, u_k)$ is bounded.
  \item If $y_{k_j}$ is a subsequence such that $y_{k_j} \rightarrow y_{\infty}$, then $y_{\infty} \in C$ and the distance between $- N_C(y_{\infty}) - \nabla \objfunc(y_{\infty})$ and $- 1/n\sum_{j=1}^n \{ \nabla \objfunc(y_{k_j}) + u_{k_j} \}$ goes to zero.
  \item $\objfunc$ is a Lyapunov function for the DI~\eqref{eq:di}.
  \item The set $\objfunc(\cZ)$ is of empty interior.
\end{enumerate}
Then, $\objfunc(y_k)$ converges and $\limsup_{k \rightarrow + \infty} d(y_k, \cZ) = 0$.
\end{theorem}
\begin{proof}
Apply \cite[Theorem 3.2]{dav-dru-kak-lee-19}, with $G = -\nabla \objfunc - N_{C}$ and $\phi = \objfunc$.
\end{proof}

\begin{lemma}
  We can replace the 4-th assumption in \Cref{theo:dav_drus} by the following assumption: if $(y_{\infty},u_{\infty})$ is a cluster point of $(y_k, u_k)$, then $u_{\infty} \in  N_{C}(y_{\infty})$.
\end{lemma}
\begin{proof}
 If $(y_{k_j})$ is a subsequence such that $y_{k_j} \rightarrow y_{\infty}$, then $1/n \sum_{j=1}^{n}- \nabla \objfunc(y_{k_j}) \rightarrow_{n \to \infty} - \nabla \objfunc(y_{\infty})$. Furthermore, for any $m \geq 0$, we can write:
 \begin{equation}
   \frac{1}{n} \sum_{j=1}^n u_{k_j} = \frac{1}{n}\sum_{j=1}^m u_{k_j} + \frac{n-m}{n}\left(\frac{1}{n-m}\sum_{j=m+1}^{n}u_{k_j} \right) \, .
 \end{equation}
 By the Caratheodory theorem, we can write $1/(n-m) \sum_{j=m}^n u_{k_j} = \sum_{i=1}^{d+1} \lambda_{m,n,i} u_{m, n,i}$, where $\lambda_{m,n,i} \geq 0$, $\sum_{i=1}^{d+1} \lambda_{m,n,i} = 1$ and $u_{m,n,i} \in \{ u_{k_{m+1}}, \dots, u_{k_n}\}$.
 Denote $\mathcal{C} \subset N_{C}(y_{\infty})$ the set of cluster points of the sequence $u_{k_j}$.
Since the sequence $(u_k)$ is bounded, for each $i \in \{1, \dots, d+1\}$, we can extract a convergent sequence from $( \lambda_{m,n,i},u_{m,n,i})$ that converges to $(\lambda_m(i),u_m(i))$, with $u_m(i) \in \mathcal{C}\cup \bigcup_{j=m+1}^{+\infty} \{u_{k_j} \}$. Thus, $1/n \sum_{j=1}^n u_{k_j} \rightarrow \sum_{i=1}^{d+1} \lambda_m(i)u_m(i)$. As a consequence, we can write:
\begin{equation}
  \lim_{n \rightarrow + \infty} 1/n \sum_{j=1}^{n} u_{k_j} = \lim_{m \rightarrow \infty} \sum_{i=1}^{d+1}\lambda_{m}(i) u_m(i) \, .
\end{equation}
For each $i \in \{1, \dots, d+1\}$, the sequences $(\lambda_m(i))_{m \geq 0}, (u_m(i))_{m \geq 0}$ are bounded. Therefore, up to an extraction of a subsequence, we can assume that they converge to some $\lambda(i), u(i)$. Notice that $u(i) \in \mathcal{C} \subset N_{C}(y_{\infty})$. Therefore,  $1/n \sum_{j=1}^{n}u_{k_j}$ converges to a convex combination of elements of $N_{C}(y_{\infty})$. By convexity of $N_{C}(y_{\infty})$ this implies that $1/n \sum_{j=1}^n u_{k_j}$ converges to an element of $N_{C}(y_{\infty})$.

\end{proof}

The following lemma provides a condition under which $\objfunc(\cZ)$ has an empty interior.
\begin{lemma}\label{lm:sard_part}
  Assume that $\objfunc :\rset^d \rightarrow \rset$ is $d$-times continuously differentiable and that $C = [a_1, b_1] \times \dots \times [a_d, b_d]$, where for $1\leq i \leq d$, $a_i, b_i$ are some real numbers. Consider $\cZ = \{ y \in \rset^d: 0 \in - \nabla \objfunc(y) - N_{C}(y)\}$. It holds that $\objfunc(\cZ)$ is of empty interior.
\end{lemma}
\begin{proof}
Denote $\mathring{C}$ as the interior of $C$. The fact that $\objfunc(\cZ \cap \mathring{C})$ has an empty interior is a consequence of Sard's theorem and the fact that $\objfunc$ is $d$-times differentiable (see \cite{sard42}). We now show that the image of $\objfunc$ of any $m$-dimensional boundary of $C$ intersected by $\cZ$ also has an empty interior. Consider $m > 0$, and fix $m-d$ coordinates of $C$ as $c_{m+1}, \dots, c_{d}$, where $c_i$ is equal to $a_i$ or $b_i$, and denote $C_m = (a_1, b_1) \times (a_2, b_2) \dots \times (a_m, b_m) \times \{ c_{m+1}\} \times \dots \times \{c_{d}\}$. Note that if $y \in \cZ \cap C_m$, then the $m$ first coordinates of $\nabla \objfunc(y)$ are zero. Thus, if we call $\objfunc_m$ the restriction of $\objfunc$ to $C_m$, then $\objfunc_m : C_m \rightarrow \rset$ is $d$ times differentiable and $\cZ \cap C_m$ is included in its set of critical points. Applying Sard's theorem to $\objfunc_m$, we obtain that $\objfunc(\cZ \cap C_m)$ has an empty interior. Since $C$ can be written as a union of these $C_m$, this completes the proof.
\end{proof}

\subsection{A proof of \Cref{th:conv_alg}}
First we need to prove that the cluster point of the iterates $\param_\infty$ belongs to the constraints set $C_\epsilon$.
\begin{lemma}\label{lm:ineq_bound}
Under assumptions of Theorem~\ref{th:conv_alg} it holds that $\limsup_{k \to \infty} \dist(\param_k,C_{\epsilon})=0$ almost surely.
\end{lemma}
\begin{proof}[Sketch of proof]
The detailed proof is given in the following section \Cref{prf:lm_ineq_bound}. The main idea is that for any $i\in \{1, \dots, d\}$, if $\psi_{\epsilon}^i(\param_k^i) < 0$ (i.e. $\param_k^i$ is outside of the constraints), then $\param_k^i$ is constantly pushed to the closest interval. Thus, the non-convergence might happen if and only if $\param_k^i$ visits one of the interval infinitely often. However, due to the fact, that $\gamma_k$ decreases to zero and that $\nabla \objfunc_j$ is bounded, this implies, for $k$ large enough, that $\param_k^i$ will never leave the ``region of attraction" of this interval (it will be kept at a distance of order $\gamma_k$ to this interval) and thus converge to it.
\end{proof}
\begin{proof}[Proof of \Cref{th:conv_alg}]
  Our goal is to apply Theorem~\ref{theo:dav_drus} and, hence, verify its assumptions. By a standard Martingale argument it holds that the sequence $\sum_{j=0}^{k} \gamma_j \eta_{j+1}$, almost surely, converges to a finite random variable (a short proof of this result is given in \Cref{prf:lm_ineq_bound}).
Consider a realization for which $\sum_{j=0}^{\infty} \gamma_j \eta_{j+1} < \infty$. Let $(\param_{\infty}, u_{\infty})$ be a cluster point of the sequence $(\param_k, u_k)$ and let $(k_j)_{j \geq 0}$ be a subsequence such that $\lim_{j \rightarrow + \infty}(\param_{k_j}, u_{k_j}) =(\param_{\infty}, u_{\infty})$. \Cref{lm:ineq_bound} shows that $\param_{\infty} \in C_{\epsilon}$. Since $\sup_{k \geq k_{0, \epsilon}}|\lambda_k^i| < + \infty$, we can extract a subsequence from $k_j$, and assume that $\lambda_{k_j} \rightarrow \lambda$. Thus, $u_{\infty}^i = - \lambda^i \psi'_{\epsilon}(\param^i_{\infty})$. Since all of the $\lambda_{k_j}^i$ are positive, it holds that $\lambda^i \geq 0$. Moreover, notice that if  $\psi_{\epsilon}(\param_{\infty}^i) >0$, then, for $j$ large enough, $\psi_{\epsilon}(\param_{k_j}^i) >0$ and, therefore, $\lambda_{k_j}^i = 0$. Hence, for $i \notin \ActiveSet(\param_{\infty})$, $\lambda^i = 0$.
This shows $u_{\infty} \in N_{C_{\epsilon}}(\param_{\infty})$. As shown in Lemma~\ref{lm:mfcq_lyap}, $\objfunc$ is a Lyapunov function for the DI: $\dot{\sy}(t) \in - \nabla \objfunc(\sy(t)) - N_{C_{\epsilon}}(\sy(t))$. In \Cref{lm:sard_part} we show that $\objfunc(\cZ)$ is of empty interior. Thus, with the help of \Cref{lm:ineq_bound,lm:mart_conv}, the assumptions of \Cref{theo:dav_drus} are satisfied, which concludes the proof.
\end{proof}

\subsection{A martingale result and proof of Lemma~\ref{lm:ineq_bound}}\label{prf:lm_ineq_bound}

We first establish a result on the convergence of the weighted sequence of perturbations.
\begin{lemma}\label{lm:mart_conv}
  Assume \Cref{assum:sum-objective}-\Cref{assum:stepsize}.  Then, almost surely, $\sum_{j=0}^{k} \gamma_j \eta_{j+1}$ converges.
\end{lemma}
\begin{proof}
  Denote by $\mathcal{F}_k$ the filtration generated by $\{ \param_1, \dots, \param_k\}$. It holds that $\PE[\widehat{\nabla} \objfunc(\param_{k})|\mathcal{F}_k] = \nabla \objfunc(\param_k)$. Furthermore, almost surely, $\PE[\|\eta_{k+1}\|^2|\mathcal{F}_k] \leq 2 \PE[\|\widehat{\nabla} \objfunc(\param_{k})\|^2| \mathcal{F}_k] +2 \|\nabla \objfunc(\param_k)\|^2 < 4 M_{\objfunc} $, where $M_{\objfunc} = \sup_{ 1\leq j \leq N}M_{\objfunc_j}$.
  Thus, for $i \in \{1, \dots, d\}$, $\sum_{j=0}^k \gamma_j \eta^i_{j+1}$ is a martingale with an almost surely bounded square variation (since $\sum_{j=0}^{+\infty} \gamma_j^2 < + \infty$). The proof is concluded by applying \cite[Theorem 11.14]{klenke2013probability}.
\end{proof}
In all the sequel, it is implicitly assumed that $\epsilon$ was chosen small enough to satisfy the assumption of Theorem~\ref{th:conv_alg}.
Denote by $k_{0, \epsilon}$ the smallest integer after which we do not perform the clipping step in Algorithm~\ref{alg:detailed}.
\begin{align}
  k_{0, \epsilon}:= \inf\{ & k \geq 0: \textrm{ for } m \geq k\, , s^c_{\epsilon, \alpha}(\widehat{\nabla} \objfunc(\param_m), \param_m) =\\
  & s_{\epsilon, \alpha}(\widehat{\nabla} \objfunc(\param_m), \param_m)\} \, .
\end{align}
Since $\limsup \dist(\param_k, C_{\epsilon}) = 0$, it holds that $\liminf \psi_{\epsilon}(\param_k^i) \geq 0$ and, therefore, $k_{0, \epsilon}$ is almost surely finite. Thus, for $k \geq k_{0, \epsilon}$, $v_k^i = [s_{\epsilon, \alpha}(\widehat{\nabla} \objfunc(\param_k), \param_k)]^i$, which implies:
\begin{equation}
 v_k^i = - \widehat{\nabla}_i \objfunc(\param_k) + \lambda_k^i \psi_{\epsilon}'(\param_k^i)\, ,
\end{equation}
with $\lambda_k^i = 0$ if $\psi_{\epsilon}(\param_k^i) >0$ and $\lambda_k^i = (v_k^i + \widehat{\nabla}_i \objfunc(\param_k))/\psi'_{\epsilon}(\param_k^i)$ otherwise. Notice that since the sequences $(v_k), (\param_k)$ are almost surely bounded, $\sup_{k \geq k_{0, \epsilon}}|\lambda_k^i|$ is almost surely finite.

 \begin{lemma}
   Assume \Cref{assum:sum-objective}-\Cref{assum:stepsize}.
   For $i \in \{1, \dots, d\}$, and for $k \geq k_{0, \epsilon}$, $\lambda_k^i \geq 0$.
 \end{lemma}
 \begin{proof}
   First, notice that if $\psi_{\epsilon}(\param_k^i) >0$, then $\lambda_k^i = 0$ by construction. Consider now the case where $\psi_{\epsilon}(\param_k^i) \leq 0$. If $- \widehat{\nabla}_i \objfunc(\param_k) \psi_\epsilon'(\param_k^i) \geq - \alpha \psi_{\epsilon}(\param_k^i)$, then $v_k^i = - \widehat{\nabla}_i\objfunc(\param_k)$ and, since for $k \geq k_{0, \epsilon}$, $\psi_{\epsilon}'(\param_k^i) \neq 0$, this implies $\lambda_k^i = 0$.
   Otherwise, $v_k^i = - \alpha \psi_{\epsilon}(\param_k^i)/\psi_\epsilon'(\param_k^i)$ and  $0 < - \alpha \psi_{\epsilon}(\param_k^i) + \widehat{\nabla}_i \objfunc(\param_k) \psi'(\param_k^i)$. Dividing the last inequality by $\{\psi_{\epsilon}'(\param_k^i)\}^2$, we obtain $0 <(-\alpha \psi_{\epsilon}(\param_k^i) +\widehat{\nabla}_i \objfunc(\param_k) \psi'_{\epsilon}(\param_k^i)) /\{\psi_{\epsilon}'(\param_k^i)\}^2 = (v_k^i + \widehat{\nabla}_i \objfunc(\param_k))/\psi'_{\epsilon}(\param_k^i) = \lambda_k^i  $.
 \end{proof}

The rest of this section is devoted to the proof of Lemma~\ref{lm:ineq_bound}.

  Denote $M= \max(M_{\epsilon, c},\sup_{ 1 \leq j \leq N} M_{\objfunc_j})$ and notice that for any $k \geq 0$ and $i\in \{1, \dots, d\}$, $\|\widehat{\nabla} \objfunc(\param_k)\| \leq M$ and $|v_k^i| \leq M$. Therefore, $|\param_{k+1}^i - \param_k^i| \leq \gamma_k M$. The lemma will be proved by the following claims.

\textit{Claim 1. For $i \in \{1, \dots, d\}$, and for $2 \leq j \leq K_i-1$ if the set $[(c^i_j + c^i_{j-1})/2, (c^i_j + c^i_{j+1})/2)$ is visited by $\param_k^i$ infinitely often, then there is $k_0$ such that for all $k > k_0$, $\param_k^i \in [(c^i_j + c^i_{j-1})/2, (c^i_j + c^i_{j+1})/2) $.}

Indeed, fix such a $j$ and denote $[c_{-}, c_{+}]$ the set $C_{\epsilon}^i \cap [(c^i_j + c^i_{j-1})/2, (c^i_j + c^i_{j+1})/2)$, where $C_{\epsilon}^i$ is the projection of $C_{\epsilon}$ onto the $i$-th coordinate. Define $k_0 = \sup \{ k : \gamma_k M \geq \max( c_{-} - (c^i_j + c^i_{j-1})/2, (c^i_j + c^i_{j+1})/2 - c_{+}) \}$. Consider $k \geq k_0$, if $ (c^i_j + c^i_{j-1})/2 \leq \param_k^i \leq c_{-} $ (we are on the left side of the interval), then the iterate is pushed to the right and $\param_k^i \leq \param_{k+1}^i$. Furthermore, by definition of $k_0$, it holds that $\param_{k+1}^i \leq c_{-} + \gamma_k M \leq (c^i_j + c^i_{j+1})/2$. This implies, that in this case $\param_{k+1}^i$ stays in $[(c^i_j + c^i_{j-1})/2, (c^i_j + c^i_{j+1})/2)$. Otherwise, if $ c_{+}\leq \param_{k}^i < (c^i_j + c^i_{j+1})/2$ (we are on the right side of the interval), then, we are pushed to the left, and, by a similar reasoning, $\param_{k+1}^i \in [(c^i_j + c^i_{j-1})/2, (c^i_j + c^i_{j+1})/2)$. Finally, if $\param_k^i \in [c_{-}, c_{+}]$, then by the way $k_0$ was defined we obtain that $\param_{k+1}^i \in [(c^i_j + c^i_{j-1})/2, (c^i_j + c^i_{j+1})/2) $. Thus, we have shown that for $k \geq k_0$, if $\param_{k}^i$ is in $[(c^i_j + c^i_{j-1})/2, (c^i_j + c^i_{j+1})/2)$, then for all $k' \geq k$, the same will be true for $\param_{k'}^i$, which completes the proof of the claim.

The proof of the following two claims is similar to the one of Claim 1.

\textit{Claim 2. For $i \in \{1, \dots, d\}$, if the set $(-\infty, (c^i_1+c^i_2)/2)$ is visited by $\param_k^i$ infinitely often, then there is $k_0$ such that for all $k > k_0$, $\param_k^i \in (-\infty, (c^i_1+c^i_2)/2)$.}

\textit{Claim 3. For $i \in \{1, \dots, d \}$, if the set $[(c^i_{K_i -1} + c^i_{K_i})/2, + \infty)$ is visited infinitely often, then there is $k_0$ such that for all $k > k_0$, $\param_{k^i} \in [(c^i_{K_i -1} + c^i_{K_i})/2, + \infty)$.}

In the following, without loss of generality, we will assume that we are in the context of the first claim and that there is $k_0$, such that for all $k \geq k_0 $, $\param_k^i \in [(c^i_j + c^i_{j-1})/2, (c^i_j + c^i_{j+1})/2)$ (the two other cases can be treated in the exact same manner).

Denote, as previously, $[c_{-}, c_{+}]$ the set $C_{\epsilon}^i \cap [(c^i_j + c^i_{j-1})/2, (c^i_j + c^i_{j+1})/2)$, where $C_{\epsilon}^i$ is the projection of $C_{\epsilon}$ onto the $i$-th coordinate.

\textit{Claim 4. There is $k_0$, such that if there are two index $m_{+} \geq m_{-} > k_0$ such that $\param_{m_{-}} < c_{-} < c_{+} < \param_{m_{+}}$, then there is $m$, satisfying $ m_{-} \leq m \leq m_{+}$, such that $\param_{m}^i \in [c_{-}, c_{+}]$.}

Indeed, define $k_0 = \sup \{ k: \gamma_k M \geq c_{+} - c_{-} \}$. Let $m_{-}, m_{+}$ be as in the claim and consider $m = \inf \{ k \geq m_{-} : \param_{k}^i \geq c_{-}\}$. It holds that $\param_{m-1}^i < c_{-} \leq \param_{m}^i \leq \param_{m-1}^i + \gamma_k M$. Since $m \geq k_0$, this implies that $\param_{m}^i \leq c_{-} + \gamma_k M \leq c_{+}$, which proves the claim.

\textit{Claim 5. There is $k_0$, such that if there are two index $m_{-} \geq m_{+} >k_0$, such that $\param_{m_{-}} < c_{-} < c_{+} < \param_{m_{+}}$, then there is $m$, satisfying $ m_{+} \leq m \leq m_{-}$, such that $\param_{m}^i \in [c_{-}, c_{+}]$.}
The proof is the identical to the one of the previous claim.

From the fourth and fifth claims, there are only three possible behaviors of $\param_k^i$. Either, $\param_k^i$ visits $[c_{-}, c_{+}]$ infinitely often (this will be treated by the sixth claim), or for $k$ large enough, $\param_k^i$ stays at the left of $[c_{-}, c_{+}]$ (this will be treated by the seventh claim), or it stays at the right of $[c_{-}, c_{+}]$ (this will be treated by the eights claim).

\textit{Claim 6. If $\param_k^i$ visits $[c_{-}, c_{+}]$ infinitely often, then $  \limsup \param_k^i \leq c_{+}$ and $\liminf \param_k^i \geq c_{-}$.}

Notice that if $\param_{k}^i > c_{+}$, then $\param_{k+1}^i \leq \param_k^i$, and if  $\param_{k}^i \leq c_{+}$ and $\param_{k+1}^i \leq c_+ + \gamma_k M$. Thus, if $k$ is such that $\param_k^i \in [c_{-}, c_{+}]$, then $\sup_{k_1 \geq k} \param_{k_1}^i \leq c_{+} + \gamma_k M$. Letting $k$ tend to infinity, proves first part of the claim. Similarly, if $k$ is such that $\param_k^i \in [c_{-}, c_{+}]$, then $\inf_{k_1 \geq k} \param_{k_1}^i \geq c_{-} - \gamma_k M$. Letting $k$ tend to infinity proves the second part of the claim.

\textit{Claim 7. If for all $k$ large enough, $\param_k^i > c_{+}$, then $\param_k^i \rightarrow c_{+}$.}

Indeed, in this case, for $k$ large enough, the sequence $\param_k^i$ is decreasing and thus has a limit. Denote this limit $\param_{+}$ and assume that $\param_{+} \neq c_{+}$, then for $k$ large enough, it holds that $\param_{k+1}^i = \param_k^i + \gamma_k v_k^i \leq \param_k^i - \gamma_k M_{+}$, where $M_{+} = \inf \{ \min(M_{\epsilon,c}, \alpha |\psi_{\epsilon}(\param)|/|\psi_{\epsilon}'(\param)| ): \param \in [\param_{+}, (c_j + c_{j+1})/2)\} >0$. Thus, for any $m$, it holds that $\param_{k+m+1}^i \leq \param_{k}^i - M_{+}\sum_{i=0}^{m} \gamma_{k+i}$. Since $\sum_{j=0}^{+\infty}\gamma_j = +\infty$, this shows that this case is impossible. Hence, $\param_k^i \rightarrow c_{+}$.

\textit{Claim 8. If for all $k$ large enough, $\param_k^i < c_{-}$, then $\param_k^i \rightarrow c_{-}$.}

Similarly, to the previous claim, for $k$ large enough the sequence $\param_k^i$ is increasing and thus has a limit. If $\param_{-} \neq c_{-}$, then for $k$ large enough and $m \geq 0$, it holds that $\param_{k+m+1}^i \geq \param_k^i + M_{-}\sum_{i=0}^m \gamma_{k+i}$, where $M_{-} = \inf \{ \min(M_{\epsilon,c}, \alpha |\psi_{\epsilon}(\param)|/|\psi_{\epsilon}'(\param)|) : \param \in ( (c_{j-1} + c_{j})/2, \param_{-}] \} >0$. Since $\sum_{j=0}^{+\infty}\gamma_j = +\infty$, this implies that $\param_{-} \neq c_{-}$ is impossible.  Hence, $\param_k^i \rightarrow c_{-}$.

These claims show that for every $i \in \{1, \dots, d\}$, $\liminf\psi_{\epsilon}(\param^i_k) \geq 0$. Therefore, $\limsup \dist(\param_k, C_{\epsilon}) = 0$.
\section{Numerical results}
\label{sec:app:numerical-results}
In this section, we give more details about our experiments, and present results on new tasks.
\subsection{Toy convex example}
We give more results about the toy example detailed in \Cref{sec:experiments}. We only compare \Algo\ and BinaryConnect \cite{Courbariaux2015} in a logistic regression problem, but we test several settings to highlight the strengths of \Algo\ : all methods are trained for a longer time ($50$ epochs) using the SGD optimizer, the learning rate is set to $1$, and gradients are calculated on random batches of $100$ or $1000$ samples. Note the rest of the experimental setting is identical: we generate $n=6000$ feature vectors $\{x_k\}_{k=1}^n$ in dimension $d=10$ drawn independently from the uniform distribution in $[-1,1]$. We randomly choose an optimal vector $w_*$ on the vertices of the hypercube and generate the labels  as follows: $y_k \sim \operatorname{Bernoulli}(\{1+\rme^{-x_k^\top  w_*}\}^{-1})$.
For completeness, we study how a full precision SGD converges to the optimal point $w_*$ of this convex problem.
\begin{figure}[h!]
    \centering
    \includegraphics[scale=0.45]{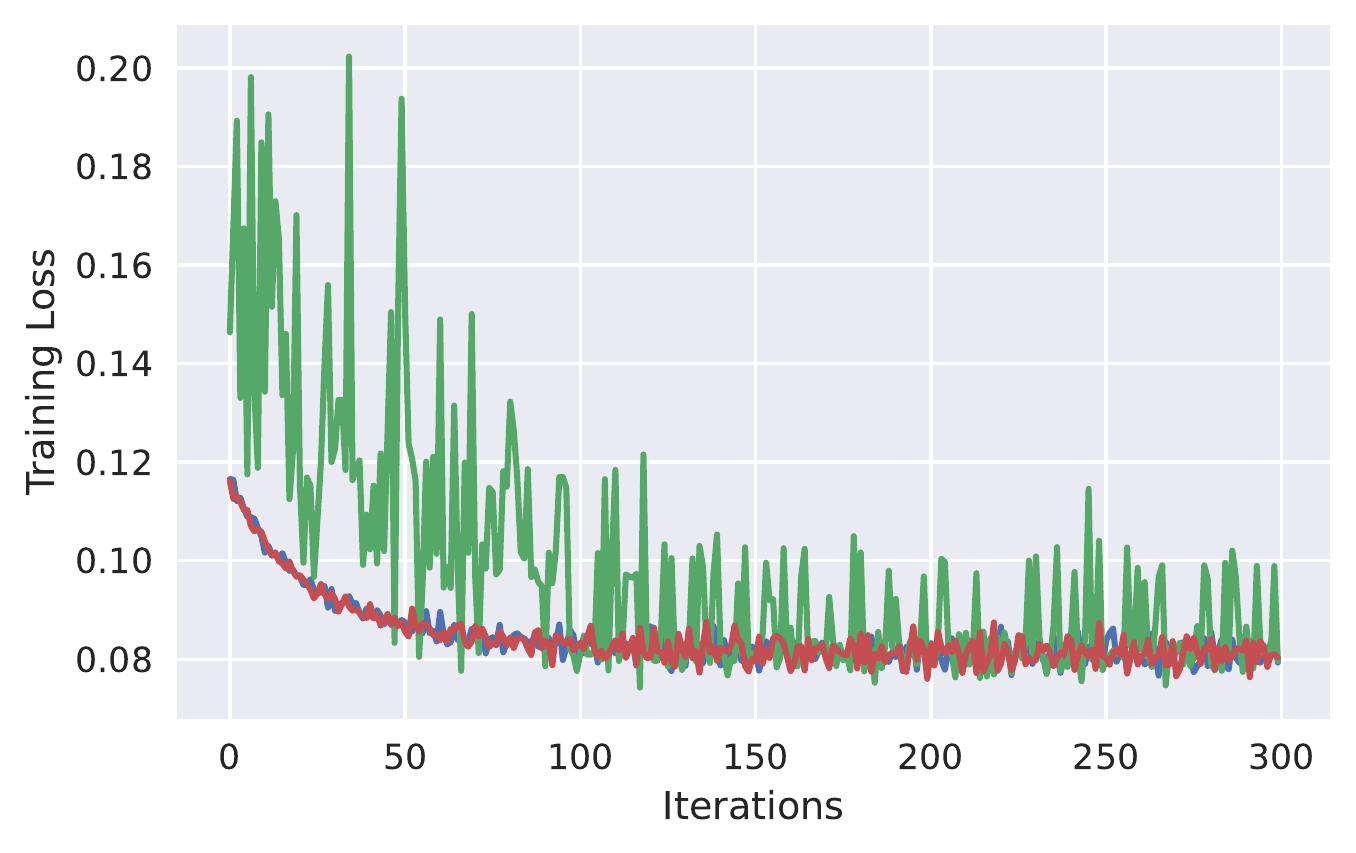}
    \includegraphics[scale=0.45]{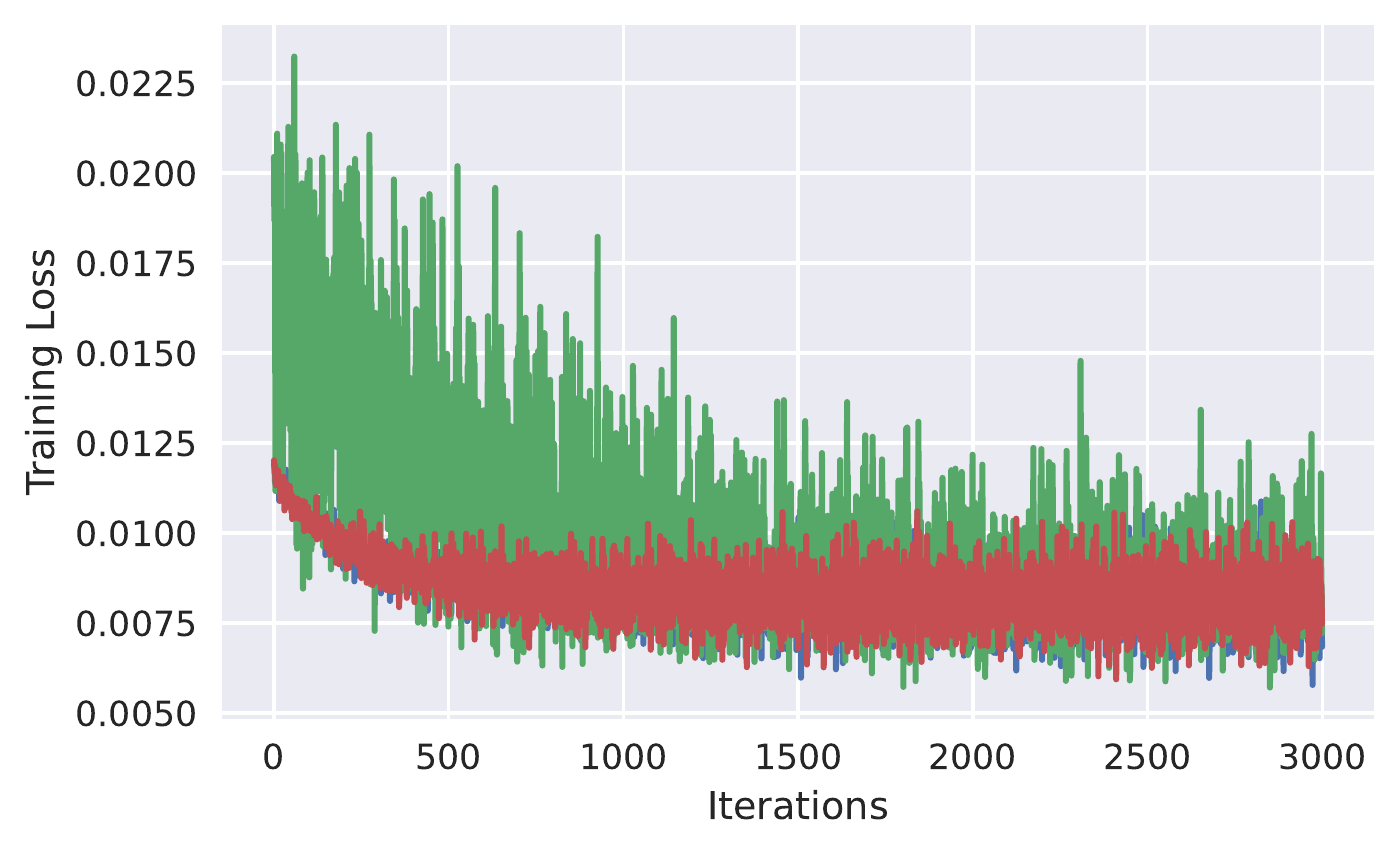}
    \caption{Training losses for the logistic regression problem with batches of size $1000$ (up panel) and $100$ (down panel).  BinaryConnect - green -  \Algo\ - blue -  full Precision methods - red -. The x-axis represents the iteration index.}
\label{fig:logistic_regression}
\end{figure}
The same conclusions can be drawn: \Algo\ is very close to the full precision method while STE method suffers from oscillations. Note however that decreasing the batch size seems to have a beneficial effect for STE, the larger variance helps to reduce the gap between STE and the other methods  (see down panel in \Cref{fig:logistic_regression}).
\subsection{"2 moons" example}
We consider the binary classification problem "2 moons dataset" presented in \Cref{sec:experiments} and inspired by \cite{meng2020training}. The training dataset consists of $n=2000$ samples and 200 test samples and is displayed in \Cref{fig:twomoons_dataset}.
A BNN with 9 weights is trained with one-hot coding and logistic loss. This BNN uses ReLu activations and its architecture is shown in \Cref{fig:network}. Four gradient-based approaches - a full precision NN, BinaryConnect, AdaSTE, and \Algo\ - are compared to exhaustive search.
In the latter, all $2^9$ binary configurations on the training and test sets are compared. \Cref{fig:histogram_bruteforce} shows that different configurations lead to near-optimal performance. It is worth noting that permutation invariance implies that many solutions are equivalent in this simple example.
\begin{figure}[h!]
    \begin{subfigure}[b]{0.45\linewidth}
        \centering

    \includegraphics[scale=0.45]{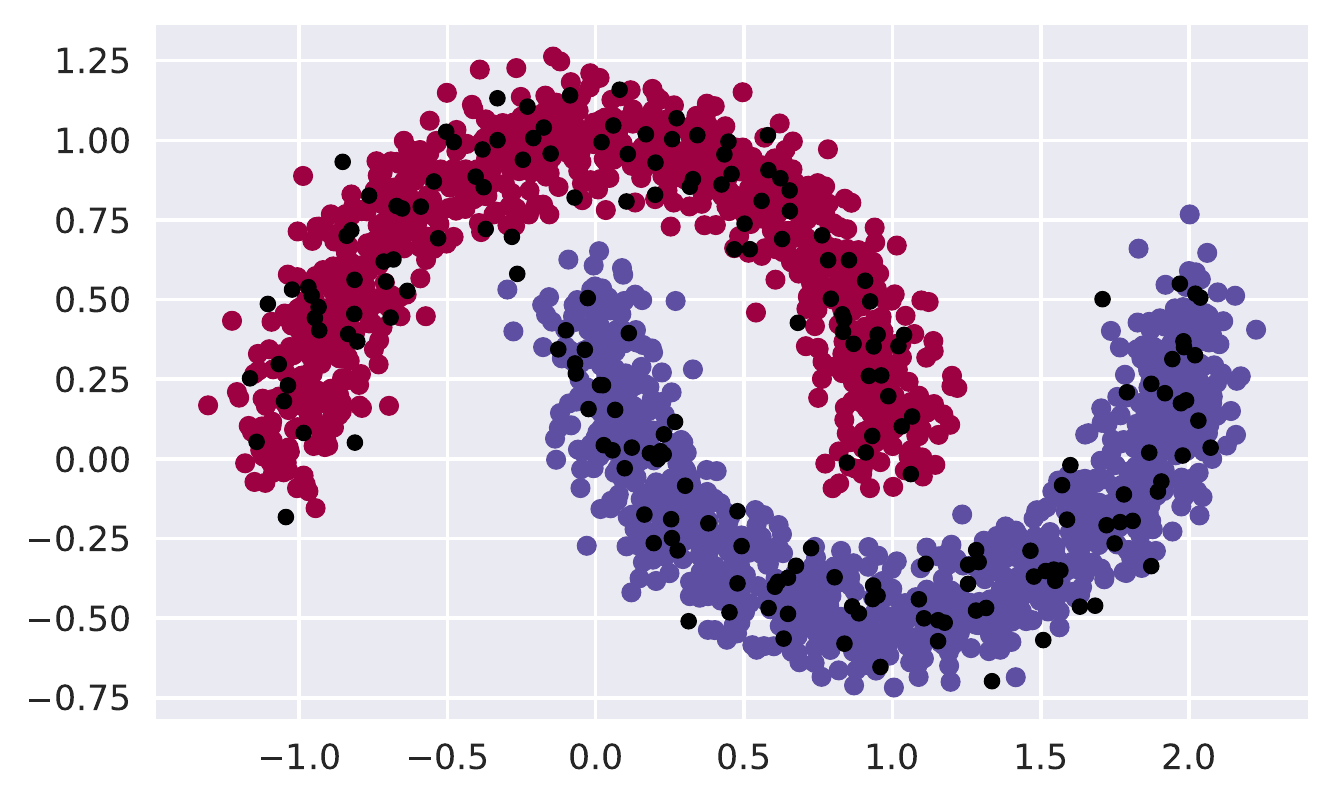}
        \caption{Two moons dataset} \label{fig:twomoons_dataset}
  \end{subfigure}\\
      \tikzset{%
  every neuron/.style={
    circle,
    draw,
    minimum size=.5cm
  },
  neuron missing/.style={
    draw=none,
    scale=2,
    text height=0.13cm,
    execute at begin node=\color{black}$\vdots$
  },
}
\begin{subfigure}[b]{0.45\linewidth}
    \centering

\begin{tikzpicture}[x=1.0cm, y=0.8cm, >=stealth]

\foreach \m/\l [count=\y] in {1,2}
  \node [every neuron/.try, neuron \m/.try] (input-\m) at (0,1.-\y) {};

\foreach \m [count=\y] in {1,2,3}
  \node [every neuron/.try, neuron \m/.try ] (hidden-\m) at (2,2-\y*1.25) {};

\foreach \m [count=\y] in {1}
  \node [every neuron/.try, neuron \m/.try ] (output-\m) at (4,.5-\y) {};

\foreach \l [count=\i] in {1,2}
  \draw [<-] (input-\i) -- ++(-1,0)
    node [above, midway] {$input_\l$};

\foreach \l [count=\i] in {1}
  \draw [->] (output-\i) -- ++(1,0)
    node [above, midway] {$output$};

\foreach \i in {1,2}
  \foreach \j in {1,2,3}
    \draw [->] (input-\i) -- (hidden-\j);

\foreach \i in {1,2,3}
  \foreach \j in {1}
    \draw [->] (hidden-\i) -- (output-\j);

\foreach \l [count=\x from 0] in {Input, Hidden, Ouput}
  \node [align=center, above] at (\x*2,2) {\l \\ layer};

\end{tikzpicture}
\caption{Basic BNN structure in dimension $d=9$} \label{fig:network}
\end{subfigure}
    \caption{2D Dataset and the associated BNN.}
\end{figure}
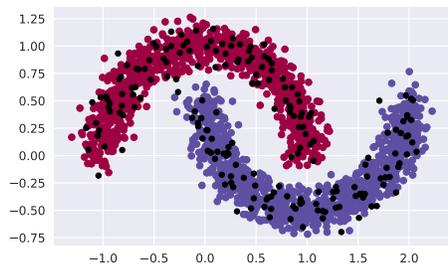
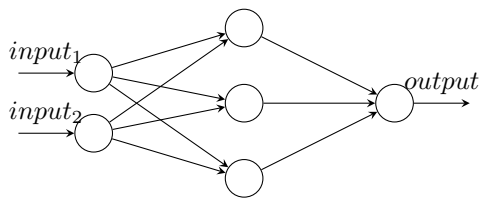
\begin{figure}[h!]
    \centering
    \includegraphics[scale=0.5]{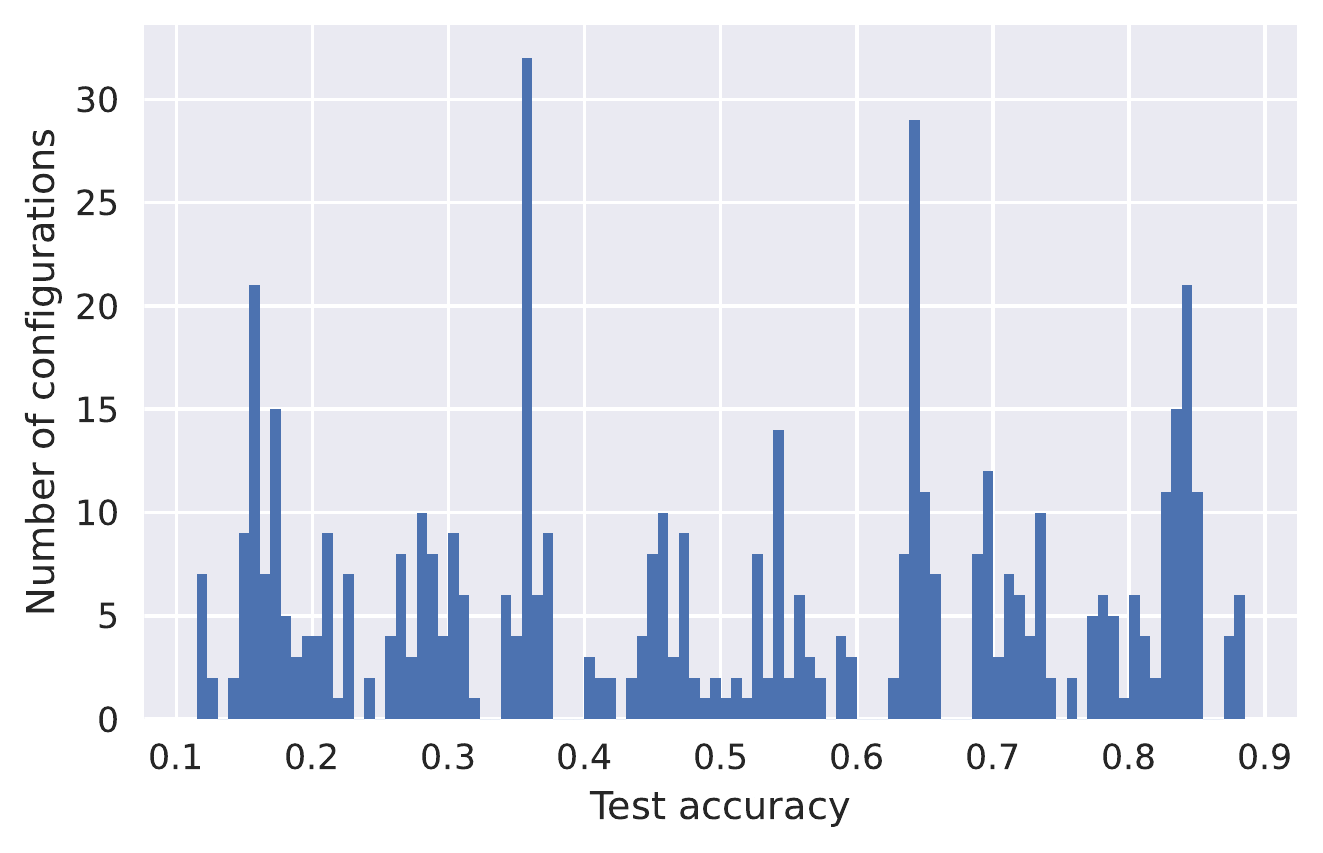}
    \caption{Histogram of test accuracies for the exhaustive search in dimension $d=512$}
    \label{fig:histogram_bruteforce}
\end{figure}
\begin{table}[t]
  \caption{Logistic Loss on test samples (average over 50 random experiments) for the binary classification problem in dimension $d=512$ after 100 steps.}
  \label{tab:loss_twomoons}
  \centering
  \begin{scriptsize}
  \begin{tabular}{cc}
    \toprule
    Method & Loss ($\times 10^{-3}$) \\
    \midrule
    Full-precision NN [W32/A32]& $2.045\pm 0.005$ \\
    \midrule
    BinaryConnect [W1/A32]& $2.32 \pm 0.11$ \\
    AdaSTE [W1/A32]& $2.24 \pm 0.10$ \\
    \midrule
    \Algo\ [W1/A32]& \textbf{$2.11$} $\pm 0.01$\\
    Exhaustive search [W1/A32]& \textbf{2.1} \\
    \bottomrule
  \end{tabular}
  \end{scriptsize}
\end{table}

\subsection{Deep learning experiments}
The performances reported in \Cref{sec:experiments} were obtained with the best combination of hyperparameters that we tested. Other combinations are listed in \Cref{tab:accuracy_cifar10_other} for the CIFAR-10 dataset.
\begin{table}[h!]
  \caption{Best Test accuracy after 100 training epochs on CIFAR-10.}
  \label{tab:accuracy_cifar10_other}
  \centering
  \begin{scriptsize}
  \begin{tabular}{cccccc}
    \toprule
    $\alpha$ / $lr$ & 0.5/0.06 & 0.2/0.01 & 0.2/0.03 & 0.2/0.05 & 0.4/0.01 \\
    \midrule
    \Algo\ & 88.51 & 85.60 & 88.42 & 88.32 & 84.50\\
    \bottomrule
  \end{tabular}
  \end{scriptsize}
\end{table}The performances reported in \Cref{tab:accuracy_cifar10} are still dependent on hyperparameter grid search and could be further improved if more resources are available.

\subsection{ImageNet with binary weights}
In this section, we compare the performance of \Algo\ [W1/A32] on a large dataset and compare it to BinaryConnect \cite{Courbariaux2015,hubara2016binarized}, Mirror Descent \cite{ajanthan2021mirror}, AdaSTE \cite{le2021adaste}, a standard full-precision NN, and a hypersphere-projected full-precision NN. To ensure a fair comparison, we compare the different methods using the same NN architecture. Moreover, we do not add bias in any layer, but introduce batch normalisation (without learning parameters) after each layer. The last connected layer is kept in full precision - a standard practice in BNN -. Contrary to \citep{liu2020reactnet, chmiel2021logarithmic}, we have kept the first convolutional layer binary. We do not use layerwise scalar contrary to \cite{Rastegari2016}.

We use a training setting similar to \Cref{sec:experiments}. We have adapted the code of \cite{ajanthan2021mirror, le2021adaste} to Resnet-18 for ImageNet. The hyperparameters for AdaSTE and MD are those prescribed for TinyImageNet. We use the same default data normalizations as the methods we compare to: we resize the input images to $256 \times 256$ and then randomly crop them to $224 \times 224$ while centering them to the appropriate sizes during training. Standard multiclass cross entropy loss is used. All models are fine-tuned for $100$ epochs using the Adam \citep{kingma2014adam} optimizer with dynamics of $0.9$ and $0.999$ and a batch of size $512$. The full precision NN is trained with an initial learning rate of $0.08$. The projected full precision NN uses a projected gradient algorithm. The same hyperparameters are used as in the "simple" algorithm, except that a deterministic projection onto the hypersphere is performed at each iteration. The \Algo\ method is described in \Cref{alg:detailed}, and we have set $\alpha$ to $0.5$. The precision threshold $\epsilon$ is decreased from epoch to epoch: it is set to $1$ at the beginning and then exponentially annealed to $.88^t$ in the last 50 epochs, where $t$ is the epoch.

Just as with \Cref{sec:experiments}, we apply the function $\operatorname{sign}(\cdot)$ to our NN before evaluating it on the test set. Each method was randomly initialized and independently executed once (due to ImageNet's longer training time). The learning rate at epochs $[20, 40]$ is divided by $2$ for all methods.
\begin{table}[h!]
  \caption{Best Test accuracy after 100 training epochs.}
  \label{tab:accuracy_imagenet}
  \centering
  \begin{tabular}{ccc}
    \toprule
    Method & \multicolumn{2}{c}{ImageNet (ResNet-18)} \\
    & Top-1 & Top-5\\
    \midrule
    Full-precision & 66.39 & 95.32 \\
    BinaryConnect & 45.85 & 71.05 \\
    MD & 46.38 & 71.18 \\
    AdaSTE & 35.37 & 62.22 \\
    Projected gradient & 2.58 & 7.93 \\
    \Algo\ & \textbf{46.95} & \textbf{72.11} \\
    \bottomrule
  \end{tabular}
\end{table}
This task is more difficult than TinyImageNet's, but we get the same result: \Algo\ outperforms all current baselines. Moreover, \Algo\ yields good  results even when trained from scratch, compared to methods \cite{bai2018proxquant, liu2020reactnet} that require fine-tuning using a pre-trained network.

\subsection{BNN with binary activations}
BNN with binary weights and binary activations offer significant time savings in inference. We applied our training procedure \Algo\ [W1/A1] to a VGG-small with $\operatorname{sign(\cdot)}$ activations instead of ReLu activations to enable inference with only XNOR and bit-counting operations. The quantization of activations is here too extreme to apply the same procedure as in \Cref{sec:experiments}. The biased quantizer SAWB from \cite{choi2018bridging} does not work anymore (empirically). We assume the loss of neural gradient information is too important when activations are quantized on 2 levels.

During the training phase, a batch normalisation layer is inserted before each sign activation to scale the variance. In the backward pass, the derivative of $\operatorname{sign(\cdot)}$ is approximated by the derivative of the function $\operatorname{tanh(\cdot)}$. During inference, we can get rid of the batch normalization (only the empirical mean is conserved and added to the bias term) and compute only binary operations.

We compare test accuracy with the CIFAR-10 dataset, which consists of 50000 training images and 10000 test images (in 10 classes). BNNs are fine-tuned for $100$ epochs using the Adam optimizer with a dynamic range of $0.9$ and $0.999$ and a batch size of $100$ with a learning rate of $0.03$. The best test classification accuracies of binary networks obtained with \Algo\ are listed in \Cref{tab:accuracy_cifar10_bnn} for different values of $\alpha \in [0.2, 0.5, 0.7]$.
\begin{table}[h!]
  \caption{Best Test accuracy after 100 training epochs.}
  \label{tab:accuracy_cifar10_bnn}
  \centering
  \begin{tabular}{cccc}
    \toprule
    $\alpha$ & 0.2 & 0.5 & 0.7 \\
    \midrule
    \Algo\ [W1/A1] & 81.12 & \textbf{84.34} & 82.92 \\
    \bottomrule
  \end{tabular}
\end{table}
The preliminary results reported in \Cref{tab:accuracy_cifar10_bnn} show that \Algo\ is state-of-the art  for training BNNs with binary weights and activations. Activation with the function $\operatorname{sign(\cdot)}$ leads to a loss in expressive power and consequently a loss in performance. Several works introduce additional tricks such as real scaling factors \cite{Rastegari2016} to bridge the gap between binary signals and their real counterparts. These tricks can be easily implemented in our approach \Algo. \end{document}